\documentclass[]{interact}

\usepackage{epstopdf}

\usepackage[numbers,sort&compress]{natbib}
\bibpunct[, ]{[}{]}{,}{n}{,}{,}
\renewcommand\bibfont{\fontsize{10}{12}\selectfont}

\usepackage{graphicx}
\usepackage{xcolor}
\usepackage[normalem]{ulem}
\usepackage{todonotes}
\usepackage{caption}
\usepackage{subcaption}

\usepackage{hyperref}
\hypersetup{
	colorlinks=true,       
	linkcolor=blue,
	citecolor=red,
	filecolor=magenta,      
	urlcolor=cyan           
}

\usepackage{cleveref}

\usepackage{hhline}

\usepackage{multirow}
\usepackage{colortbl} 

\usepackage{algorithm}
\usepackage{algorithmic}

\usepackage{cancel}

\theoremstyle{plain}
\newtheorem{theorem}{Theorem}[section]
\newtheorem{lemma}[theorem]{Lemma}

\newtheorem{proposition}[theorem]{Proposition}

\crefname{theorem}{Theorem}{Theorem}
\crefname{lemma}{Lemma}{Lemma}
\crefname{corollary}{Corollary}{Corollary}
\crefname{proposition}{Proposition}{Proposition}

\theoremstyle{definition}
\newtheorem{definition}{Definition}[section]

\theoremstyle{remark}
\newtheorem{remark}{Remark}[section]
\newtheorem{notation}{Notation}[section]

\numberwithin{algorithm}{section}

\newcommand{\N}{\ensuremath{\mathbb{N}}}
\newcommand{\R}{\ensuremath{\mathbb{R}}}
\newcommand{\E}{\ensuremath{\mathbb{E}}}

\renewcommand{\v}{\ensuremath{\boldsymbol}}	
\newcommand{\revtwo}[1]{#1}

\newcommand{\blue}[1]{#1}

\begin{document}


\title{\blue{ATE-SG: Alternate Through the Epochs Stochastic Gradient for Multi-Task Neural Networks}}

\author{
\name{Stefania Bellavia\textsuperscript{a}\textsuperscript{b}\thanks{CONTACT: S. Bellavia. Email: stefania.bellavia@unifi.it},
Francesco {Della Santa}\textsuperscript{c}\textsuperscript{b},  
Alessandra Papini\textsuperscript{a}\textsuperscript{b}
}
\affil{
\textsuperscript{a}Dipartimento di Ingegneria Industriale, Universit\`a degli Studi di Firenze, Florence, Italy;\\
\textsuperscript{b}Gruppo Nazionale per il Calcolo Scientifico INdAM, Piazzale Aldo Moro 5, 00185, Rome, Italy;\\
\textsuperscript{c}Dipartimento di Scienze Matematiche, Politecnico di Torino, Turin, Italy
}
}

\maketitle

\begin{abstract}
This paper introduces novel alternate training procedures for hard-parameter sharing Multi-Task Neural Networks (MTNNs). Traditional MTNN training faces challenges in managing conflicting loss gradients, often yielding sub-optimal performance. The proposed alternate training method updates shared and task-specific weights alternately \blue{through the epochs}, exploiting the multi-head architecture of the model. This approach reduces computational costs \blue{per epoch and memory requirements}.  Convergence properties similar to those of the classical stochastic gradient method are established. 
Empirical experiments demonstrate \blue{enhanced training regularization}
and reduced computational demands. In summary, our alternate training procedures offer a promising advancement for the training of hard-parameter sharing MTNNs.
\end{abstract}




\section{Introduction}\label{sec:intro}

Multi-Task Learning (MTL) consists of jointly learning multiple tasks rather than
individually, in such a way that the knowledge obtained by learning a task can be exploited for learning other tasks, hopefully improving the generalization performance of all the tasks at hand \cite{MTL-survey1}.
For the case of Neural Networks (NNs), MTL is approached by building NN architectures characterized by multiple output layers, one for each task, connected to (at least) one shared input layer; then, Multi-Task NNs (MTNNs) are characterized by an inherent layer sharing property and, historically, can be divided into \emph{hard-parameter sharing} MTNNs and \emph{soft-parameter sharing} MTNNs \cite{MTL-survey}. In this work, we focus on the hard-parameter sharing case, i.e., on MTNNs characterized by a so-called multi-head design architecture, where a first block of shared layers connects the inputs to multiple task-specific blocks of layers (see \Cref{fig:multitaskNN_example}). Summarizing, the idea behind these MTNNs is to build a shared encoder that branches out into multiple task-specific decoders \cite{MTL-survey} (e.g., see \cite{Kokkinos_2017_CVPR,multinet_2018,Jagannath2021_MTL,Jagannath2021_arXiv,Bragman2019}).

The NN model is generally trained to simultaneously make predictions for all tasks, where the loss is a weighted sum of all the task-specific loss functions ({\it aggregate loss function}) \cite{MTL-survey1}. This approach presents several difficulties. Descent directions of different loss functions at the current iterate
may conflict and the direction used to update the NN parameters may produce an increase of a single loss despite the aggregate loss function decreases. Then, this approach often yields lower performance than its corresponding single-task counterparts \cite{PascalMichiardi2021_MTLalt}. 
Different approaches have been proposed by researchers to overcome these difficulties and obtain more robust procedures. These approaches can be divided into two main types: 
$i$) modify the training procedure by
considering also the gradients of the task-specific losses, and/or by training all or a subset of the weights with respect to single tasks,
e.g. see \cite{Strezoski_2019_ICCV,Maninis2019,PascalMichiardi2021_MTLalt,MTL-bargain,vicentejota2023}; 
$ii$) adaptively choose a good setting of the weights in the aggregate loss function for a good balance of the magnitudes of the task-specific losses, e.g. see \cite{MTL_UncWeightLoss,GradNorm,MTL_DynamicTaskPrior,vicenteannor2021,vicentecms2022,mercierdesideri2018}. Among these, papers \cite{vicenteannor2021, mercierdesideri2018} deal with a general multiobjective problem, of which the MTNN training is a special case, and aim at approximating the entire Pareto front. In order to adaptively compute the aggregate loss weights these approaches require the solution of a minimization subproblem at each iteration.

Further, we also recall the Block Coordinate Descent (BCD) method, where the parameters of the MTL model are partitioned and updated with respect to the corresponding subproblems (see \cite{MTL-survey1} and the references there-in); especially in non-Deep Learning MTL problems, such a kind of approach is useful to reduce the complexity of the optimization learning problem (e.g., see \cite{LIU2009_BCDmultitask}).

\subsection{Contribution}
In this paper we propose a novel approach for training a generic hard-parameter sharing MTNN. Though inspired both by the approaches based on task-specific gradients and by the BCD method, it is distinguished by the following characteristics.
\begin{itemize} 
    \item We always aim at reducing the aggregate loss function, but rather than alternating among stochastic gradient steps for a single task as in \cite{vicentejota2023,PascalMichiardi2021_MTLalt,MTL-bargain} we alternate stochastic estimators of the gradient with respect to the shared NN parameters and stochastic estimators of the gradient with respect to the task-specific NN parameters. The 
    shared and task-specific parameters are then updated alternately; in this way, when the task-specific weights are updated, all specific-task losses are reduced simultaneously. This important property depends on the multi-head architecture that characterizes the type of MTNN we consider in this work.
    \blue{We stress that we do not need to solve optimization subproblems as in the multi-gradient method in \cite{vicenteannor2021, mercierdesideri2018}.}
    
    \item Our approach is theoretically well-founded. We consider the case of nonconvex, differentiable functions and analyze both the case where shared and task-specific parameters are alternately updated at each iteration, 
    and the case where we \blue{keep updating} the shared (task-specific) parameters of the NN for one or more epochs, and then alternate, \blue{i.e. we alternate through the epochs. We show that the convergence properties of the classical stochastic gradient method are maintained. We carried out the convergence analysis assuming Lipschitz continuity of the partial gradients with respect to both shared and task-specific parameters. This yields to the use of potentially larger 
    sequences of learning rates.}

    \item The alternate training we present is a new stochastic gradient
    training procedure for hard-parameter sharing MTNNs, that compared with the classical stochastic gradient approach both reduces memory requirements and computational costs, and \blue{allows to use different sequences of learning rates in
    the shared and task-specific updating phases. Further, numerical experiments show that our approach regularizes the training phase}.
    
    \item One of our objectives was to devise easy-to-apply procedures and to provide a ready-to-use version of our proposed training routine for MTNNs (see \Cref{app:implATESG}), implementable within the most used Deep Learning frameworks in literature (e.g., see \cite{keras2015, tensorflow2015-whitepaper}).
\end{itemize}

The content of this work is organized as follows. We start by introducing the MTNNs and analyzing the properties of gradients computed with respect to shared or task-specific weights (\Cref{sec:multitaskNN}). Then, we describe the new alternate training method, and discuss its convergence properties (\Cref{sec:alterante_training}). After that, a section of numerical experiments (\Cref{sec:num_exp}) illustrates a comparison between MTNNs trained classically and trained using the proposed alternate training. Finally, conclusions about advantages and properties of the proposed method are summarized (\Cref{sec:conclusion}).

\section{Multiple-Task Neural Networks}\label{sec:multitaskNN}

A hard-parameter sharing MTNN for a MTL problem made of $K\in\N$ tasks is an NN model with an architecture characterized by: one main block of layers, called \emph{trunk}, connected to the input layer(s); $K$ independent blocks of layers, called \emph{branches}, connected to the last layer of the trunk. The last layer of the $k$-th branch is the output layer of the MTNN for the $k$-th task, for each $k=1,\ldots,K$.

The main idea behind this type of architecture (see \Cref{fig:multitaskNN_example}) is that the trunk encodes the inputs, learning the new representation characterized by features important for all the $K$ tasks. Then, each branch reads this representation (i.e., the output of the last trunk's layer) and decodes it independently from other branches, learning its task. In other words, we can interpret the output of a hard-parameter sharing MTNN as a concatenation of $K$ independent decoding operations applied to an encoding operation applied to the same input signals.

In the following, we formalize the definition of hard-parameter sharing MTNN and the observation about the outputs of an MTNN. From now on, for simplicity, we take for granted that when we talk about MTNNs we are considering a hard-parameter sharing MTNN as in the next definition.

\begin{definition}[Hard Parameter Sharing Multi-Task Neural Network]\label{def:MTNN}
Let $\mathrm{N}$ be an NN with characterizing function $\widehat{\v{F}}:\R^n\times\R^p\rightarrow \R^{m_1}\times\cdots\times\R^{m_K}$, where the domain $\R^n\times \R^p$ represents the Cartesian product between the space of the NN inputs ($\ \R^n$) and the space of the NN trainable parameters ($\ \R^p$). 
Then, $\mathrm{N}$ is a \emph{hard parameter sharing multi-task NN} (MTNN) with respect to $K$ tasks in $\R^{m_1},\ldots ,\R^{m_K}$, respectively, if $\mathrm{N}$'s architecture is characterized by $K+1$ smaller NNs, $\mathrm{N}_0,\ldots ,\mathrm{N_K}$, such that:
\begin{enumerate}
    \item the characterizing function of $\mathrm{N}_0$ is a function $\widehat{\v{F}}_0:\R^n\times\R^{p_0}\rightarrow\R^{m_0}$;
    \item the characterizing function of $\mathrm{N}_k$ is a function $\widehat{\v{F}}_k:\R^{m_0}\times\R^{p_k}\rightarrow\R^{m_k}$, for each $k=1,\ldots ,K$;
    \item $\mathrm{N}$ is obtained by connecting the output layer of $\mathrm{N}_0$ to the first layers of   $\mathrm{N}_1,...\,,\mathrm{N}_K$.
\end{enumerate}
In particular, we define $\mathrm{N}_0$ as the architecture's block shared by the $K$ tasks, while $\mathrm{N}_k$ is defined as the architecture's block specific of task $k$, for each $k=1,\ldots ,K$. 

Let $\v{w}_k\in \R^{p_k}$ be the vector of trainable parameters (i.e., weights and biases) of $\mathrm{N}_k$, for each $k=0,\ldots ,K$: the parameters in $\v{w}_0$ are defined as \emph{shared parameters} of $\mathrm{N}$, while the parameters in $\v{w}_k$ of $\mathrm{N}_k$ are defined as \emph{task-specific parameters} of $\mathrm{N}$ with respect to task $k$, for each $k=1,\ldots ,K$.
Then, $\sum_{k=0}^K p_k=p$ and $\widehat{\v{F}}$ is such that
\begin{equation*}
    \widehat{\v{F}}(\v{x}; \v{w}) = 
    \begin{bmatrix}
        \widehat{\v{F}}_1(\widehat{\v{F}}_0(\v{x};\v{w}_0);\v{w}_1)\\
        \vdots\\
        \widehat{\v{F}}_K(\widehat{\v{F}}_0(\v{x};\v{w}_0);\v{w}_K)
    \end{bmatrix}
\end{equation*}
for each $\v{x}\in\R^n$, where $\v{w}=(\v{w}_0^T,\ldots ,\v{w}_K^T)^T\in\R^p$.
\end{definition}

\begin{figure}[htb]
    \centering
    \includegraphics[width=0.35\textwidth]{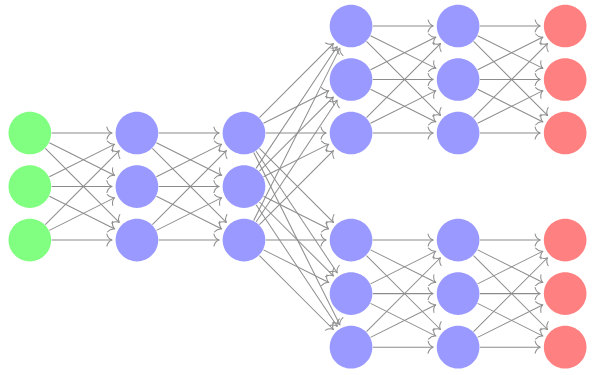}
    \caption{Example of MTNN with $2$ tasks. On the left half of the figure, there is $\mathrm{N}_0$ (with input layer in green); on the right half of the figure, there are $\mathrm{N}_1$ and $\mathrm{N}_2$ (with output layers in red).}
\label{fig:multitaskNN_example}
\end{figure}

\subsection{Loss Differentiation and Multiple Tasks}\label{sec:loss_diff}

In MTL, the loss function of a model typically is a weighted sum of different losses, evaluated with respect to each task. In Notation \ref{not:Bk_lossk} below, we introduce the symbols and the formalization we use to describe the aggregated loss and the task-specific losses of an MTNN, as well as the batches of corresponding input-output pairs. In this work, we always assume that the task-specific loss functions of the MTNN, and hence the aggregated loss also, are differentiable functions with respect to all the NN parameters.

\begin{notation}[Aggregated and task-specific batches and losses]\label{not:Bk_lossk}
Let $\mathrm{N}$ be an MTNN as in \Cref{def:MTNN} and let $\mathcal{B}\subset \R^n\times \R^m$ be a batch of input-output pairs for ${\rm N}$, where $m=\sum_{k=1}^K m_k$; a batch is always assumed finite and non-empty. Then, we introduce the following notations:
\begin{enumerate}
    \item we denote by $\mathcal{B}^k\subset \R^n\times\R^{m_k}$ the batch of input-output pairs related to the $k^{\rm th}$ task of ${\rm N}$ obtained from the batch $\mathcal{B}$; i.e.:
    \begin{equation*}
        \mathcal{B}^k := \left\lbrace (\v{x}, \v{y}_k ) \in\R^n\times\R^{m_k} \left . \right| (\v{x},\v{y})\in\mathcal{B} \right\rbrace,\quad k=1 ,\ldots ,K,
    \end{equation*}
    where  $\,\v{y}_k=(y_1^{(k)},\ldots ,y_{m_k}^{(k)})^T$ and $\,\v{y}=(\v{y}_{1}^T,\ldots ,\v{y}_{K}^T)^T\in\R^m$;
    \item we denote by $\ell_k:\mathrm{P}^*(\R^n \times\R^{m_k})\times \R^{p_0+p_k}\rightarrow\R$ a loss function defined for the task $k$ of ${\rm N}$, for each $k=1 ,\ldots ,K$, where $\mathrm{P}^*(A)$ denotes the set of finite and non-empty subsets of $A$, for each set $A$. For example, assuming $\ell_k$ as the \emph{Mean Square Error} (MSE), we have that
    \begin{equation*}
        \ell_k(\mathcal{B}^k;\v{w}_0,\v{w}_k) = 
        \frac{1}{|\mathcal{B}^k|} \sum_{(\v{x},\v{y}_k)\in\mathcal{B}^k} \left( \widehat{\v{F}}_k(\widehat{\v{F}}_0(\v{x};\v{w}_0);\v{w}_k) - \v{y}_k\right)^2\,,
    \end{equation*}
    for each batch $\mathcal{B}^k\subset\R^n \times\R^{m_k}$;
\item  
we denote by $\ell:\mathrm{P}^*(\R^n \times\R^m)\times\R^p\rightarrow\R$ the aggregated loss function of $\mathrm{N}$, such that $\ell$ is a linear combination with positive coefficients of the task-specific losses $\ell_1,\ldots ,\ell_K$:
\begin{equation}\label{eq:total_loss}
    \ell(\mathcal{B};\v{w}) := \sum_{k=1}^K \lambda_k \ell_k(\mathcal{B}^k;\v{w}_0,\v{w}_k )\,,
    \quad {\rm with }~~\lambda_1,\ldots ,\lambda_k\in\R^+.
\end{equation}
\end{enumerate}    
\end{notation}

Our alternate training method takes inspiration both from the BCD method and from those methods based on exploiting the task-specific gradients (see \Cref{sec:intro}). Indeed, it is easily seen that the gradient of $\ell$ with respect to the task-specific parameters of task $k$ is equal to the gradient of $\lambda_k \ell_k$, i.e.:
\begin{equation*}
    \nabla_{\v{w}_k}\ell(\mathcal{B};\v{w}) = \lambda_k \nabla_{\v{w}_k}\ell_k(\mathcal{B}^k;\v{w}_0, \v{w}_k)\,,
\end{equation*}
for each $k=1,\ldots ,K$, and for any batch $\mathcal{B}\in\mathrm{P}^*(\R^n\times \R^m)$. 
Then, once observed this characteristic, it is almost immediate to notice also that the gradient of $\ell(\mathcal{B};\v{w})$ with respect to all the task-specific parameters is just a concatenation of the $K$ gradients of the task-specific losses with respect to their own task-specific parameters (and multiplied by the coefficients), namely:
\begin{equation*}
    \nabla_{\v{w}_{\rm ts}}\ell(\mathcal{B};\v{w}) =
    \begin{bmatrix}
    \lambda_1 \nabla_{\v{w}_1}\ell_1(\mathcal{B}^1;\v{w}_0,\v{w}_1)\\
    \vdots\\
    \lambda_K \nabla_{\v{w}_K}\ell_K(\mathcal{B}^K;\v{w}_0,\v{w}_K)
    \end{bmatrix}
    \,,
\end{equation*}
for each $\mathcal{B}\in\mathrm{P}^*(\R^n\times \R^m)$,
where $\v{w}_{\rm ts}\in\R^{p_{\rm ts}}$ is the vector denoting the concatenation of all the task-specific parameters; i.e.:
\begin{equation*}
    \v{w}_{\rm ts} :=
    (\v{w}_1^T,\ldots ,\v{w}_K^T)^T
    \in\R^{p_{\rm ts}}
    \,, \quad {\rm with } ~~ p_{\rm ts}=\sum_{k=1}^K p_k.
\end{equation*}

As a consequence of these results, we 
\blue{ state in \Cref{prop:grads_and_descdirs} the descent properties  of  the two anti-gradients of 
$\ell$ with respect to the shared parameters and the task-specific parameters, respectively, for any batch $\mathcal{B}$}.
The proof of the proposition is omitted, since it is trivial.

\begin{proposition}[Gradients and Descent Directions]\label{prop:grads_and_descdirs}

Let $\mathrm{N}$ be an MTNN as in \Cref{def:MTNN} and let $\ell,\ell_1,\ldots ,\ell_K$ be the losses in \eqref{eq:total_loss}.
Let $\v{w}\in\R^p$ be the vector of trainable parameters of $\mathrm{N}$; specifically, $\v{w}$ is the concatenation of the shared parameters $\v{w}_0\in\R^{p_0}$ and the task-specific parameters $\v{w}_{\rm ts}\in\R^{p_{\rm ts}}$: 
\begin{equation*}
    \v{w} =  
   \begin{bmatrix}
   \v{w}_0\\
    \v{w}_{\rm ts}
   \end{bmatrix}
    \,.
\end{equation*}

Then, for each fixed batch $\mathcal{B}$ the vectors 
\begin{equation}\label{eq:descent_dirs_grads}
    \begin{bmatrix}
    -\nabla_{\v{w}_0}\ell(\mathcal{B};\v{w})\\
    \v{0}
    \end{bmatrix}
    \,, \
    \begin{bmatrix}
    \v{0}\\
    -\nabla_{\v{w}_{\rm ts}}\ell(\mathcal{B};\v{w})
    \end{bmatrix}
    \in
    \R^p
\end{equation}
are descent directions for the loss $\ell(\mathcal{B};\v{w})$ at $\v{w}$, 
if $\nabla_{\v{w}_{0}}\ell(\mathcal{B};\v{w}) \neq \v{0}\in\R^{p_0}$ and $\nabla_{\v{w}_{\rm ts}}\ell(\mathcal{B};\v{w})\neq\v{0}\in\R^{p_{\rm ts}}$, respectively. Moreover, $(\v{0}^T,-\nabla^T_{\v{w}_{k}}\ell_k(\mathcal{B}^k;\v{w}_0,\v{w}_k))^T$ (subvector of the \blue{second} vector in \eqref{eq:descent_dirs_grads}) is also a descent direction for $\ell_k(\mathcal{B}^k;\v{w}_0,\v{w}_k)$, for each $k=1,\ldots ,K$.

\end{proposition}

We conclude this section remarking the different implications of the two descent directions \eqref{eq:descent_dirs_grads}:
\begin{itemize}
    \item $\v{w}_0$-based direction: it is a descent direction for $\ell(\mathcal{B};\v{w})$ that updates the shared parameters only. 
    It allows \blue{reducing} the loss with respect to the weights that affect all the tasks, then there are no guarantees of reducing all the task-specific losses too;
    \item $\v{w}_{\rm ts}$-based direction: is a descent direction for $\ell(\mathcal{B};\v{w})$ but also for all the task-specific losses $\ell_1,\ldots ,\ell_K$, and it updates the task-specific parameters only. 
    It allows to reduce both the main loss and the task-specific losses with respect to the weights $\v{w}_1,\ldots ,\v{w}_{K}$ that affect only the losses $\ell_1,\ldots ,\ell_K$, respectively (i.e., only the corresponding tasks).\\
\end{itemize}
\blue{   
Starting from these properties, in the next section we formulate new
alternate training procedures 
and prove their convergence properties assuming Lipschitz continuity of the gradients.
Specifically, 
there exist   shared  Lipschitz constants $L_0$, $L_0^{\rm ts}$, and
 task-specific Lipschitz constants $L_{\rm ts}$, $L_{\rm ts}^0$, such that for any  
$    \v{w} =  ( \v{w}_0^T,\,\v{w}_{\rm ts}^T ) ^T
    \,,
$
$    \v{d} =  ( \v{d}_0^T,\,\v{d}_{\rm ts}^T ) ^T
    \in \R^p
$:

\begin{equation} \label{lipschitz} 
\begin{split}
a)~~~& \Vert \nabla_{\v{w}_0} \ell(\,\cdot\,;\,\v{w})-\nabla_{\v{w}_0} \ell(\,\cdot\,;\,\v{w+d})\Vert 
~~\le~ \, L_0 \,\Vert \v{d_0}\Vert, ~~~ {\rm when} ~~ 
    \v{d}_{\rm ts}= 
    \v{0}, \\
b)~~~&\Vert \nabla_{\v{w}_0} \ell(\,\cdot\,;\,\v{w})-\nabla_{\v{w}_0} \ell(\,\cdot\,;\,\v{w+d})\Vert 
~~\le~ \, L_0^{\rm ts}\, \Vert \v{d}_{\rm ts}\Vert, ~~ {\rm when} ~~ 
    \v{d}_0= 
    \v{0},\\
c)~~~&\Vert \nabla_{\v{w}_{\rm ts}} \ell(\,\cdot\,;\,\v{w})-\nabla{\v{w}_{\rm ts}} \ell(\,\cdot\,;\,\v{w+d})\Vert  ~\le~ \,
L_{\rm ts} \,\Vert \v{d_{\rm ts}}\Vert, ~~ {\rm when} ~~ 
    \v{d}_0 =  
   \v{0}    \,,\\
d)~~~&\Vert \nabla_{\v{w}_{\rm ts}} \ell(\,\cdot\,;\,\v{w})-\nabla{\v{w}_{\rm ts}} \ell(\,\cdot\,;\,\v{w+d})\Vert  ~\le~ \,
L_{\rm ts}^0 \,\Vert \v{d_{0}}\Vert, ~~~ {\rm when} ~~ 
    \v{d}_{\rm ts} =  
   \v{0}.    
    \end{split}
\end{equation} 
It can be easily seen that~ $\max\,\{L_0,L_0^{\rm ts},L_{\rm ts},L_{\rm ts}^0\}\le L$ 
(see e.g. \cite{Wright2015})), where $L$ denotes the  standard Lipschitz constant of the gradient:}
$$
\Vert \nabla \ell(\,\cdot\,;\,\v{w})-\nabla \ell(\,\cdot\,;\,\v{w}+\v{d})\Vert \le \,
L \Vert \v{d}\Vert.
$$
Where not explicitly specified, as above, gradients are taken with respect to all the parameters $\v{w}$.

\section{Alternate Training}\label{sec:alterante_training}

 \blue{The idea of an alternate training can be realized in many different ways. In this work, we define alternate training methods which are modifications of a classical Stochastic Gradient (SG) procedure. They can also be extended to other optimization procedures, but we deserve these generalizations for future work. }

\Cref{prop:grads_and_descdirs} defines two alternative descent directions for the loss function.
Devising a training procedure based on the alternate usage of these directions or their stochastic estimators may yield the following practical advantages.

\blue{
\begin{enumerate}
    \item\label{rem:adv_alt_train_memory} 
    \emph{Alternate training for reduced memory usage.} The advantage of the alternate training concerning memory is almost evident in situations where 
    $p_0 \not \approx p$ and $p_{\rm ts} \not \approx 0$ (or vice-versa). 
    Indeed, at each step, an alternate procedure requires the storage of a gradient $\nabla_{\v{w}_0}\ell$ with dimension $p_0$ or a gradient $\nabla_{\v{w}_{\rm ts}}\ell$ with dimension $p_{\rm ts}$; in both cases, the gradient has dimension smaller than the ``global'' gradient computed with respect to all the NN's weights 
    (i.e., with dimension $p=p_0+p_{\rm ts}$). 
 
 \item\label{rem:adv_alt_train_cost} 
    \emph{Alternate training for reduced computational cost.}
    We observe that the computation of $\nabla_{\v{w}_{\rm ts}}\ell$ is cheaper than the computation of both $\nabla_{\v{w}_0}\ell$ and $\nabla_{\v{w}}\ell$, since only the task-specific layers of the NN are involved during the back-propagation \cite{Rumelhart1986_BACKPROP_Nature}; then, 
    at equal number of epochs, training the model
    with a procedure that alternates  between  $\nabla_{\v{w}_0}\ell$ and
    $\nabla_{\v{w}_{\rm ts}}\ell$ to update the current iterate is less costly than the classical stochastic gradient which always uses $\nabla_{\v{w}}\ell$.
    
    \item\label{rem:adv_alt_train} 
    \emph{Alternate training for regularization.}  Relying on the $\v{w}_{\rm ts}$-based direction partially reduces the typical difficulty of MTL models about selecting a direction that is not a descent direction for all the tasks, and therefore reduces the possibility of having an increase of some losses despite the overall objective function decreases (see \Cref{sec:intro}). This interesting regularization property of the training phase is illustrated in the numerical experiments of \Cref{sec:num_exp}.
\end{enumerate}
}
\noindent
In the next subsections, we define and analyze a couple of alternate training strategies. 

\subsection{Simple Alternate Training}\label{sec:SAT}
The Simple Alternate Training (SAT) method is a two steps iterative process that alternately updates $\v{w}_0$ and $\v{w}_{\rm ts}$ in an MTNN.
The variable $\v{w}_0$ is updated at each iteration using a stochastic estimator of 
$\nabla_{\v{w}_0}\ell$, while $\v{w}_{\rm ts}$ is updated using a stochastic estimator of $\nabla_{\v{w}_{\rm ts}}\ell$. In what follows we denote the training set as $\mathcal{T}$. 
We name SAT-SG this procedure in order to emphasize the relationship with SG, and we describe one iteration in \Cref{alg:SAT}.

\begin{algorithm}
\caption{SAT-SG - Simple Alternate Training SG for MTNNs}
\label{alg:SAT}
~~
\begin{description}
\item[Data:] $(\v{w}_0, \v{w}_{\rm ts})=\v{w}^{(i)}$ (current iterate for the trainable parameters), $\mathcal{T}$ (training set), $B$ (mini-batch size), 
\blue{ $\eta_i^{0}$ and $\eta_i^{\rm ts}$} 
(learning rates), $\ell=\ell(\mathcal{B};\v{w})$ (loss function).

\item[Iteration $i$:] \quad

    \begin{algorithmic}[1]
        \STATE Sample randomly a batch $\mathcal{B}_1$ from $\mathcal{T}$ s.t. $|\mathcal{B}_1|=B$

        \STATE $\v{w}_0 \gets \v{w}_0 - \blue{\eta_i^0 }\,
             \nabla_{\v{w}_0}\ell(\mathcal{B}_1;\v{w}^{(i)})$
            
        \STATE $\v{z}^{(i)} \gets  (\v{w}_0, \v{w}_{\rm ts})$
        
        \STATE Sample randomly a batch $\mathcal{B}_2$ from $\mathcal{T}$ s.t. $|\mathcal{B}_2|=B$
           
        \STATE $\v{w}_{\rm ts} \gets \v{w}_{\rm ts} - \blue{\eta_i^{\rm ts}}\,
                \nabla_{\v{w}_{\rm ts}}\ell(\mathcal{B}_2; \v{z}^{(i)})$
        
        \STATE $\v{w}^{(i+1)} \gets  (\v{w}_0, \v{w}_{\rm ts})$
        
        \RETURN $\v{w}^{(i+1)}$ \quad (updated iterate for MTNN's weights)
    \end{algorithmic}
\end{description}
\end{algorithm}

\blue{Convergence properties of SAT-SG are proven in \Cref{{teo:simp_alt_grads_conv}};  
both cases of constant and diminishing learning rates are considered. 
\begin{notation}
From now on, to shorten the notation, we will omit to explicitly indicate the dependence of the loss function from the batch of data when this coincides with the whole training set $\mathcal{T}$, namely we will write $\ell(\v{w})$ for $ \ell(\mathcal{T};\v{w})$, and $\nabla \ell(\v{w})$ for $\nabla \ell(\mathcal{T};\v{w})$.
\end{notation}
We first provide in \Cref{teo:simp_alt_grads} 
an upper bound for the expected conditioned value of the loss function $\ell(\v{w})$ for sufficiently small learning rates $\eta_i^0$ and $\eta_i^{\rm ts}$.} 

\begin{lemma}[SAT-SG]\label{teo:simp_alt_grads}
Let $\{\v{w}^{(i)}\}_{i\ge 0}, \{\v{z}^{(i)}\}_{i\ge 0} \subset \R^p$ be two sequences
generated by the SAT-SG method and \blue{$\{\eta_i^0\}_{i\ge 0}$,
$\{\eta_i^{\rm ts}\}_{i\ge 0}$,}
be the sequences of the learning rates used. Let $\mathcal A_i$ denote the $\sigma$-algebra induced by $ \v{w}^{(0)}, \v{z}^{(0)}$, $\v{w}^{(1)}, \v{z}^{(1)}$, ..., $ \v{w}^{(i)}$,
and $\mathcal A_{i+\frac{1}{2}}$ the $\sigma$-algebra induced by $ \v{w}^{(0)}, \v{z}^{(0)}$, $ \v{w}^{(1)}, \v{z}^{(1)}$, ..., $ \v{w}^{(i)}, \v{z}^{(i)}$. 

\smallskip\noindent
Assume that the batches $\mathcal{B}_1$ and $\mathcal{B}_2$ are sampled randomly and uniformly, that there exist two positive constants $M_1$ and $M_2$ such that
\begin{eqnarray} \label{Egradiente0}
\E[\Vert\nabla_{\v{w}_0} \ell(\mathcal{B};\v{w}^{(i)})\Vert^2 | \mathcal A_i]\le M_2 \, \Vert\nabla_{\v{w}_0}\ell(\v{w}^{(i)})\Vert^2+M_1\\ \nonumber \\  
\label{EgradienteTS}
\E[\Vert\nabla_{\v{w}_{\rm ts}}\ell(\mathcal{B};\v{z}^{(i)})\Vert^2 | \mathcal A_{i+\frac{1}{2}}]\le M_2 \, \Vert\nabla_{\v{w}_{\rm ts}}\ell(\v{z}^{(i)})\Vert^2+M_1
\end{eqnarray}
for any batch of data $\mathcal{B}$, and \blue{that 
$\nabla_{\v{w}_0}\ell(\,\cdot\,;\v{w})$ and  $\nabla_{\v{w}_{\rm ts}}\ell(\,\cdot\,;\v{w})$  satisfy the Lipschitz continuity
conditions (\ref{lipschitz}).

\smallskip\noindent
Then the following properties hold:
\begin{equation} \label{expect0}  
    \E[\ell(\v{z}^{(i)}) | \mathcal A_i]    ~\le ~ \ell(\v{w}^{(i)})
      -  \eta_i^0 \,G_i^0 \,\Vert\nabla_{\v{w}_0} \ell(\v{w^{(i)}})\Vert^2
        +\frac{L_0}{2}\,(\eta_i^0)^2 \, M_1,
 \end{equation}
with  $ ~G_i^0= 1-\frac{L_0}{2} \,\eta_i^0 \,M_2 > 0\,$ 
for any $\,i\ge 0\,$ and
$\,0<\eta_i^0<  \,\frac{1}{L_0} \,\frac{2}{M_2},$ 
\begin{equation}\label{expect2}
\E[\ell(\v{w}^{(i+1)}) | \mathcal A_{i+\frac{1}{2}}] ~\le~
 \ell(\v{z}^{(i)})
      -  \eta_i^{\rm ts} \,G_i^{\rm ts} \,\Vert\nabla_{\v{w}_{\rm ts}} \ell(\v{z^{(i)}})\Vert^2
        +\frac{L_{\rm ts}}{2}\,(\eta_i^{\rm ts} )^2 \, M_1,
\end{equation}
with  $ ~G_i^{\rm ts}= 1-\frac{L_{\rm ts}}{2} \,\eta_i^{\rm ts} \,M_2 > 0\,$
for any $\,i\ge 0\,$ and
$0<\eta_i^{\rm ts}<  \,\frac{1}{L_{\rm ts}}\,\frac{2}{M_2}$.
}

\medskip\noindent
Moreover, for sufficiently small values of the learning rates, e.g. 
\begin{equation}\label{eta_bounds}
0<\eta_i^0<  \,\min\,\left \{\frac{1-C}{L_0} \frac{2}{M_2}\,,\,\frac{C}{L_{\rm ts}^0}\right \}
~~{\rm and} ~~ 
0<\eta_i^{\rm ts}<   \,\min\,\left \{\frac{1-C}{L_{\rm ts}} \frac{2}{M_2}\,,\,\frac{C}{L_{\rm ts}^0}\right \}
\end{equation}
for any 
 $C\in (0,1)$ and $i\ge 0$, 
the following properties hold: 
\begin{equation}\label{sufficient_decrease}
G_i^{0}>C, \quad \quad  \quad   G_i^{\rm ts}>C, \quad \quad
\end{equation}
\vskip -20pt
\begin{equation}\label{expected:descent_property_readable}
    \E[\ell(\v{w}^{(i+1)})|\mathcal A_i]  ~  \le  ~\ell(\v{w}^{(i)})   
 - ~ \eta_i^{\rm min} \,G_i \,\Vert\nabla \ell(\v{w^{(i)}})\Vert^2
       +         M_1 \frac{ L_0\,(\eta_i^0)^2  + L_{\rm ts} \, (\eta_i^{\rm ts})^2}{2}~  
\end{equation}
with 
$ ~G_i= C- L^0_{\rm ts}\,  \eta_i^{\rm max}   > 0$, 
$\eta_i^{\rm min}=\min \{\eta_i^0,\, \eta_i^{\rm ts}\}$ and 
$\,\eta_i^{\rm max}=\max \{\eta_i^0,\, \eta_i^{\rm ts}\}.$
\end{lemma}

\begin{proof}
First, we consider the updating 
of the shared parameters $\v{w}_0$ (see steps 2 and 3 of \Cref{alg:SAT}), 
and use Taylor's theorem  and inequality (\ref{lipschitz}.a) to obtain
\begin{equation*}
    \ell(\v{z}^{(i)}) \le \ell(\v{w}^{(i)})
      ~+ ~ \eta_i^0 \,\nabla \ell(\v{w^{(i)}})^T 
      \begin{bmatrix}  -\nabla_{\v{w}_0}\ell(\mathcal{B}_1;\v{w}^{(i)}) \\  \v{0}
     \end{bmatrix}
    +\frac{L_0}{2}\,(\eta_i^0)^2 \,\Vert \nabla_{\v{w}_0}\ell(\mathcal{B}_1;\v{w}^{(i)}) \Vert^2.\\
\end{equation*}
Then, taking the conditioned expected value on both sides, 
exploiting assumption \eqref{Egradiente0} and the fact that the subsampled gradient $\nabla_{\v{w}_0}\ell(\mathcal{B}_1;\v{w}^{(i)})$ is an unbiased estimator, we have:
\begin{equation*}    \begin{split} 
    \E[\ell(\v{z}^{(i)}) | \mathcal A_i]  & ~\le ~ \ell(\v{w}^{(i)})
      -  \eta_i^0 \,\nabla_{\v{w}_0} \ell(\v{w^{(i)}})^T 
      \E[\nabla_{\v{w}_0}\ell(\mathcal{B}_1;\v{w}^{(i)})|\mathcal A_i] ~+ \\
      & \quad ~~~\frac{L_0}{2}\,(\eta_i^0)^2 \,
       \E[ \Vert \nabla_{\v{w}_0}\ell(\mathcal{B}_1;\v{w}^{(i)}) \Vert^2\, | \mathcal A_i]\\
         &   ~\le ~ \ell(\v{w}^{(i)})
      -  \eta_i^0 \,\Vert\nabla_{\v{w}_0} \ell(\v{w^{(i)}})\Vert^2
        +\frac{L_0}{2}\,(\eta_i^0)^2 \, ( M_2 \,\Vert\nabla_{\v{w}_0}\ell(\v{w}^{(i)})\Vert^2+M_1),
\end{split}
\end{equation*}
\blue{which, by setting $G_i^0=1-\frac{L_0}{2} \,\eta_i^0 \,M_2$, reduces to
\eqref{expect0}, 
with $G_i^0>0$ for sufficiently small values of $\eta_i^0$, namely 
$\eta_i^0<  \,\frac{1}{L_0}\,\frac{2}{M_2}$. }
 
Similarly, after updating the task-specific parameters $\v{w}_{\rm ts}$ we have  
 \begin{equation*}
    \ell(\v{w}^{(i+1)}) \le \ell(\v{z}^{(i)})
      ~+ ~ \eta_i^{\rm ts} \,\nabla \ell(\v{z^{(i)}})^T 
      \begin{bmatrix}  \v{0} \\ -\nabla_{\v{w}_{\rm ts}}\ell(\mathcal{B}_2;\v{z}^{(i)})   
     \end{bmatrix}
    +\frac{L_{\rm ts}}{2}\,(\eta_i^{\rm ts})^2 \,\Vert \nabla_{\v{w}_{\rm ts}}\ell(\mathcal{B}_2;\v{z}^{(i)}) \Vert^2,\\
\end{equation*}
so that, proceeding as before in conditioned expected value and using
assumption \eqref{EgradienteTS}, we get 
\blue{\eqref{expect2}
if~ $\eta_i^{\rm ts}<  \,\frac{1}{L_{\rm ts}}\,\frac{2}{M_2}$, so that
$G_i^{\rm ts}=1-\frac{L_{\rm ts}}{2} \,\eta_i^{\rm ts} \,M_2 > 0$.

\noindent
Further,    \eqref{sufficient_decrease} trivially follows from  \eqref{eta_bounds},
since $ G_i^0$
and 
$G_i^{\rm ts}$
are bounded below by a positive constant $C\in (0,1)$ for sufficiently small values of 
$\eta_i^0$ and $\eta_i^{\rm ts}$, respectively 
$\eta_i^0\le  \,\frac{1-C}{L_0}\,\frac{2}{M_2}$ and 
$\eta_i^{\rm ts}\le  \,\frac{1-C}{L_{\rm ts}}\,\frac{2}{M_2}$.
}

\smallskip\noindent
Now, using inequality (\ref{lipschitz}.d) and the definition of  $\v{z}^{(i)}$, we observe that 
\begin{equation*}\begin{split}
  \Vert\nabla_{\v{w}_{\rm ts}} \ell(\v{z^{(i)}})\Vert^2
   & 
  =  \Vert\nabla_{\v{w}_{\rm ts}} \ell(\v{w}^{(i)})  +
  \nabla_{\v{w}_{\rm ts}}\ell( \v{z}^{(i)})-
 \nabla_{\v{w}_{\rm ts}}\ell( \v{w}^{(i)})\,\Vert^2~\\
  & = \Vert\nabla_{\v{w}_{\rm ts}} \ell(\v{w^{(i)}})\Vert^2 
      \, + \Vert \, \nabla_{\v{w}_{\rm ts}}\ell(\v{z}^{(i)}) -
 \nabla_{\v{w}_{\rm ts}}\ell(\v{w^{(i)}})\,\Vert^2 +\\
 & ~~~
\,  2\, \nabla_{\v{w}_{\rm ts}}\ell(\v{w}^{(i)})^T  
    \left(\, \nabla_{\v{w}_{\rm ts}}\ell( \v{z}^{(i)})-
 \nabla_{\v{w}_{\rm ts}}\ell( \v{w}^{(i)})\,\right)
    \,\\
 &  
  \ge   \,\Vert\nabla_{\v{w}_{\rm ts}} \ell(\v{w^{(i)}})\Vert^2
\, - 2\, \Vert \nabla_{\v{w}_{\rm ts}}\ell(\v{w}^{(i)}) \Vert \, 
        \Vert \, \nabla_{\v{w}_{\rm ts}}\ell(\v{w}^{(i)})-
 \nabla_{\v{w}_{\rm ts}}\ell(\v{z^{(i)}})\,\Vert \,\\
  & 
\ge    \,\Vert\nabla_{\v{w}_{\rm ts}} \ell( \v{w}^{(i)})\Vert^2 -
 2L^0_{\rm ts}\,\eta_i^0 \, \Vert\nabla_{\v{w}_{\rm ts}} \ell( \v{w}^{(i)})\Vert \,
  \Vert\nabla_{\v{w}_{0}} \ell( \mathcal{B}_1;\v{w}^{(i)})\Vert;
  \end{split}
\end{equation*}
then using this last inequality in \eqref{expect2} we have
\begin{equation}\label{expect3}\begin{split}
\E[\ell(\v{w}^{(i+1)}) | \mathcal A_{i+\frac{1}{2}}]  ~\le~ &
 \ell(\v{z}^{(i)})
      -  \eta_i^{\rm ts} \,G_i^{\rm ts} \,\Vert\nabla_{\v{w}_{\rm ts}} \ell(\v{w^{(i)}})\Vert^2 +
      \\
      & 2L^0_{\rm ts}\,\eta_i^0 \eta_i^{\rm ts} G_i^{\rm ts} \, 
      \Vert\nabla_{\v{w}_{\rm ts}} \ell( \v{w}^{(i)})\Vert \,
      \Vert\nabla_{\v{w}_{0}} \ell(\mathcal{B}_1; \v{w}^{(i)})\Vert
      +
      \\
      & \frac{L_{\rm ts}}{2}\,(\eta_i^{\rm ts})^2 \, M_1.
\end{split} \end{equation}
Finally, recalling that
$$\E[\ell(\v{w}^{(i+1)}) | \mathcal A_{i}]=
\E[\,\E[\ell(\v{w}^{(i+1)}) | \mathcal A_{i+\frac{1}{2}}]\,|\, \mathcal A_i],$$
combining \eqref{expect0} and \eqref{expect3}, using \eqref{sufficient_decrease}
 and the upper bound $G_i^{\rm ts}<1$,
we have
{\small
\begin{equation*}\begin{split}
\E[\ell(\v{w}^{(i+1)}) | \mathcal A_i]& ~\le~
 \E [ \ell(\v{z}^{(i)}) | \mathcal A_i]
       -  \eta_i^{\rm ts} \,G_i^{\rm ts} \,\Vert\nabla_{\v{w}_{\rm ts}} \ell(\v{w^{(i)}})\Vert^2 +\\
      & \quad \quad 2L^0_{\rm ts}\,\eta_i^0 \eta_i^{\rm ts} G_i^{\rm ts} \, 
      \Vert\nabla_{\v{w}_{\rm ts}} \ell( \v{w}^{(i)})\Vert \,
  \Vert\nabla_{\v{w}_{0}} \ell( \v{w}^{(i)})\Vert
        +\frac{L_{\rm ts}}{2}\,(\eta_i^{\rm ts})^2 \, M_1\\
& ~\le~
 \ell(\v{w}^{(i)})
      -  \eta_i^0 \,G_i^0 \,\Vert\nabla_{\v{w}_0} \ell(\v{w^{(i)}})\Vert^2
        +\frac{L_0}{2}\,(\eta_i^0)^2 \, M_1  
 -  \eta_i^{\rm ts} \,G_i^{\rm ts} \,\Vert\nabla_{\v{w}_{\rm ts}} \ell(\v{w^{(i)}})\Vert^2 ~ +\\
       & \quad \quad 2L^0_{\rm ts}\,\eta_i^0 \eta_i^{\rm ts} G_i^{\rm ts} \, 
      \Vert\nabla_{\v{w}_{\rm ts}} \ell( \v{w}^{(i)})\Vert \,
  \Vert\nabla_{\v{w}_{0}} \ell( \v{w}^{(i)})\Vert
        +\frac{L_{\rm ts}}{2}\,(\eta_i^{\rm ts})^2 \, M_1\\
& ~ \le ~
 \ell(\v{w}^{(i)})
      - C\,( \eta_i^0  \,\Vert\nabla_{\v{w}_0} \ell(\v{w^{(i)}})\Vert^2
+ \eta_i^{\rm ts}  \,\Vert\nabla_{\v{w}_{\rm ts}} \ell(\v{w^{(i)}})\Vert^2)
      ~ +\\
        & \quad \quad 2L^0_{\rm ts}\,\eta_i^0 \eta_i^{\rm ts}  \, 
      \Vert\nabla_{\v{w}_{\rm ts}} \ell( \v{w}^{(i)})\Vert \,
  \Vert\nabla_{\v{w}_{0}} \ell( \v{w}^{(i)})\Vert +
        M_1 \frac{ L_0\,(\eta_i^0)^2 \, + L_{\rm ts} \,(\eta_i^{\rm ts})^2 }{2}.
        \end{split}
\end{equation*}
}

\blue{From the last inequality, using the relations 
$\eta_i^{\rm min} \le \eta_i^0, \eta_i^{\rm ts} \le \eta_i^{\rm max}$
and $  2 a b \, \le \,a^2+b^2$, we  obtain 
\small
\begin{equation*}\begin{split}
\E[\ell(\v{w}^{(i+1)}) | \mathcal A_i]& ~\le~
\ell(\v{w}^{(i)})
      - C\,\eta_i^{\rm min}\, ( \,\Vert\nabla_{\v{w}_0} \ell(\v{w^{(i)}})\Vert^2
+  \,\Vert\nabla_{\v{w}_{\rm ts}} \ell(\v{w^{(i)}})\Vert^2)
      ~ +\\
        & \quad \quad L^0_{\rm ts}\,  \eta_i^{\rm max} \eta_i^{\rm min}   \, 
      ( \,\Vert\nabla_{\v{w}_0} \ell(\v{w^{(i)}})\Vert^2
+  \,\Vert\nabla_{\v{w}_{\rm ts}} \ell(\v{w^{(i)}})\Vert^2)~ +\\
    & \quad \quad  M_1 \frac{ L_0\,(\eta_i^0)^2 \, + L_{\rm ts} \,(\eta_i^{\rm ts})^2 }{2},\\
        \end{split}
\end{equation*}
}
which reduces to  \eqref{expected:descent_property_readable}
by setting $G_i=C- L^0_{\rm ts}\, \eta_i^{\rm max}  $, with $G_i>0$ 
for $\,\eta_i^{\rm max} \,<  \,\frac{C}{L_{\rm ts}^0}$.
\end{proof}

\blue{Inequality \eqref{expected:descent_property_readable} in Lemma \ref{teo:simp_alt_grads} paths the way for the convergence Theorem \ref{teo:simp_alt_grads_conv} below,  
in the spirit of classical convergence 
results on stochastic gradient methods (see  Theorems 4.8 and 4.10 in \cite{Bottou_Nocedal}).
We will make use of the following theorem.
}

\begin{theorem}[Robbins-Sigmund 1971 \cite{ROBBINS1971}] \label{RS71}
Let $U_i, \beta_i, \xi_i,\rho_i$ be nonnegative $\mathcal A_{i}$-measurable random variables such that
$$
\E[U_{i+1} | \mathcal A_{i}]\le 
(1+\beta_i) U_i+\xi_i-\rho_i\quad \quad i=0,1,2\ldots .
$$
Then, on the set $\{\sum_i \beta_i<\infty, \sum_i \xi_i<\infty\}$, $U_i$ converges almost surely to a random variable $U$ and $\sum_i \rho_i<\infty$ almost surely. 
\end{theorem}

\blue{
\begin{theorem} [SAT-SG convergence] \label{teo:simp_alt_grads_conv}
Assume that the hypotheses in Lemma \ref{teo:simp_alt_grads} hold and that the 
loss function $\ell$ is bounded below by $\ell_{low}$.  
 Let $\{\v{w}^{(i)}\}_{i\ge 0}$ be a sequence of iterates generated by Algorithm \ref{alg:SAT} with learning rates $\{\eta_i^0\}_{i\ge 0}$ and $\{\eta_i^{\rm ts}\}_{i\ge 0}$  satisfying 
 \eqref{eta_bounds}.
Then,

\medskip
\noindent
i) for  sequences of learning rates such that 
\begin{equation}\label{eq:lr_hypotheses_minmax}
  a)~~~  \sum_{i=0}^{\infty}    \eta_i^{\rm min} =\infty\,, \qquad   \qquad b)~~~ 
    \sum_{i=0}^{\infty}  (\eta_i^{\rm max})^2 <\infty\,,   
\end{equation}
the following holds:
 \begin{equation}\label{liminf_0}
 \lim\inf_{i\to \infty}  ||\nabla  \ell(\v{w}^{(i)})||^2=0\quad\quad a.s.;
 \end{equation}
 
\noindent
ii) for fixed learning rates, $\eta_i^0=\eta_0$ and $\eta_i^{\rm ts}=\eta_{\rm ts}$
for all $i\ge 0$,
    the average-squared gradients of $\ell$ satisfy
    the following inequality at any iteration $J \in \mathbb{N}$:
 \begin{eqnarray} \label{fixed_learning}
\E[\frac{1}{J+1}\sum_{i=0}^{J}\Vert\nabla\ell(\v{w}^{(i)})\Vert^2]      \le&
      M_1& \frac{ L_0\,\eta_0^2 \, + L_{\rm ts} \,\eta_{\rm ts}^2 }{2\, G\, 
     \min \{\eta_0,\eta_{\rm ts} \}} 
    +\frac{ \ell(\v{w}^{(0)})-\ell_{low}}{  (J+1)\,G\, \min \{\eta_0,\eta_{\rm ts} \}}
     \label{expectedaveragesum}\\  
     \label{optimality_gap}
     \xrightarrow{J\rightarrow \infty} &  
       M_1& \frac{ L_0\,\eta_0^2 \, + L_{\rm ts} \,\eta_{\rm ts}^2 }{2\, G\,  \min \{\eta_0,\eta_{\rm ts} \} } ,
 \end{eqnarray}
with 
$G = C-L_{\rm ts}^0 \,\max \{\eta_0,\eta_{\rm ts} \}\ge C/2>0$,  
for 
$C \in (0,1)$ and $\max \{\eta_0,\eta_{\rm ts} \}\le \frac{1}{2} 
\frac{C}{ L_{\rm ts}^0}$.
\noindent
  \end{theorem}
\begin{proof}
In case i) 
we are going to  show that 
\begin{equation} \label{senza_G_i_tutto}
\sum_{i=0}^{\infty}    \eta_i^{\rm min}   \Vert\nabla  \ell(\v{w}^{(i)} )\Vert^2<\infty \quad\quad a.s.,
\end{equation}
from which \eqref{liminf_0} follows using assumption (\ref{eq:lr_hypotheses_minmax}.a).
To this aim, we first rewrite \eqref{expected:descent_property_readable} as follows:
{\small 
\begin{equation*}
    \E[\ell(\v{w}^{(i+1)})|\mathcal A_i] - \ell_{low} ~  \le  ~\ell(\v{w}^{(i)})  - \ell_{low} 
 - ~ G_i~
   \eta_i^{\rm min} \,\Vert\nabla \ell(\v{w^{(i)}})\Vert^2
       +         M_1 \frac{ L_0\,(\eta_i^0)^2  + L_{\rm ts} \, (\eta_i^{\rm ts})^2}{2}~ . 
\end{equation*}
}
Then, using Theorem \ref{RS71} with 
$\,U_i=\ell(\v{w}^{(i)})-\ell_{low}>0$, 
$\,\xi_i=M_1 \frac{ L_0\,(\eta_i^0)^2  + L_{\rm ts} \, (\eta_i^{\rm ts})^2}{2}$,
$\,\beta_i=0\,$, 
and
$\,\rho_i=\eta_i^{\rm min} G_i\Vert\nabla  \ell(\v{w}^{(i)})\Vert^2$,
 we have that the sequence $\{\ell(\v{w}^{(i)})\}$ is convergent almost surely and
\begin{equation*} \label{somme_rho}
\sum_{i=0}^{\infty}    \eta_i^{\rm min}  G_i\Vert\nabla  \ell(\v{w}^{(i)} )\Vert^2<\infty \quad\quad a.s.
\end{equation*}
Since by assumption (\ref{eq:lr_hypotheses_minmax}.b) $\eta_i^{\rm max} \to 0$,
 $G_i=C-L_{\rm ts}^0 \,\eta_i^{\rm max}$ is positive and bounded away from zero
 for all $i$ sufficiently large, so that   also \eqref{senza_G_i_tutto} holds.

\medskip
In case ii),~ since $\,\eta_i^0=\eta_0$,~ $\eta_i^{\rm ts}=\eta_{\rm ts}\,$  for all $i$, ~and 
$G_i = G = C-L_{\rm ts}^0 \,\max \{\eta_0,\eta_{\rm ts} \}>0$ by \eqref{eta_bounds}, taking the total expectation of \eqref{expected:descent_property_readable} we get
\begin{equation*}
    \E[\ell(\v{w}^{(i+1)})] - \E[\ell(\v{w}^{(i)})] \le
     -\, \min \{\eta_0, \eta_{\rm ts} \}\,  G\,\E[\Vert\nabla \ell(\v{w}^{(i)})\Vert^2]
     + M_1 \frac{ L_0\,\eta_0^2 \, + L_{\rm ts} \,\eta_{\rm ts}^2 }{2}.
\end{equation*}
Then summing up for $i=0$ to $i=J$ and recalling that $\ell$ is bounded from below by $\ell_{low}$ we obtain
\begin{equation*} \begin{split}
\ell_{low}-\ell(\v{w}^{(0)}) & ~\le~  \E[\ell(\v{w}^{(J+1)})] - \ell(\v{w}^{(0)})\\
&~\le~  -  \min \{\eta_0, \eta_{\rm ts} \}\,   G\, \sum_{i=0}^{J}  E[\Vert\nabla\ell(\v{w}^{(i)})\Vert^2] + M_1\sum_{i=0}^{J}  \frac{ L_0\,\eta_0^2 \, + L_{\rm ts} \,\eta_{\rm ts}^2 }{2}.
\end{split} \end{equation*}
By rearranging the previous inequality we finally have
$$
\E[\sum_{i=0}^{J}\Vert\nabla\ell(\v{w}^{(i)})\Vert^2]  \,\le\, 
   (J+1) M_1 \frac{ L_0\,\eta_0^2 \, + L_{\rm ts} \,\eta_{\rm ts}^2 }{2\, G\, \min \{\eta_0,\eta_{\rm ts} \}} 
  +\frac{ \ell(\v{w}^{(0)})-\ell_{low}}{ G\,\min \{\eta_0,\eta_{\rm ts} \}},
$$   
from which \eqref{expectedaveragesum} follows by dividing for $J+1$.
\end{proof}
}

\blue{
Let us assume to set $C=\frac{1}{2}$ in Theorem \ref{teo:simp_alt_grads_conv} and compare  \eqref{fixed_learning} with the corresponding result for the SG method \cite[Th. 4.8]{Bottou_Nocedal}. 
In \eqref{fixed_learning} 
we have on the right-hand side the additional factor $1/G<4$ rather than a factor $2$ as in the corresponding result for the SG method. On the other hand, we expect to be able to take larger learning rates as the Lipschitz constants involved are smaller than the Lipschitz constant of the full gradient. 
Regarding the optimality gap,  
in addition to potentially smaller Lipschitz constants, also the constant $M_1$ in \eqref{Egradiente0}-\eqref{EgradienteTS}  may be smaller than the corresponding constant used to bound the expected value of $\Vert\nabla\ell(\v{w})\Vert^2$ with SG.

We finally observe that SAT-SG can also be seen as a block stochastic gradient method and has some similarities with the approach proposed in \cite{XuYin2015}. We stress that in \cite{XuYin2015} the same batch $\mathcal{B}=\mathcal{B}_1=\mathcal{B}_2$ is used for each block and, using our notation,
the convergence analysis 
is carried out under the stronger assumption 
$$
\|\E[\nabla_{\v{w}_{\rm ts}}\ell(\mathcal{B}; \v{z^{(i)}}) -\nabla_{\v{w}_{\rm ts}}\ell(\v{z^{(i)}})| \mathcal A_i]\|~\le~A\eta_i^{\max} 
$$
where $A$ is a positive constant, at any iteration $i$.
Further, in \cite{XuYin2015} it is proved that
such an assumption is verified in case the batch  $\mathcal{B}$ is a singleton and the following Lipschitz
conditions hold:
\begin{align}
    \Vert \nabla_{\v{w}_{\rm 0}}\ell(\,\cdot\,;\,\v{w})-\nabla_{\v{w}_{\rm 0}} \ell(\,\cdot\,;\,\v{w}+\v{d})\Vert &\le \,
L \Vert \v{d}\Vert\\
 \Vert \nabla_{\v{w}_{\rm ts}}\ell(\,\cdot\,;\,\v{w})-\nabla_{\v{w}_{\rm ts}} \ell(\,\cdot\,;\,\v{w}+\v{d})\Vert &\le \,
L \Vert \v{d}\Vert.
\end{align}

}

\subsection{Alternate Training through the Epochs}
\label{sec:ATE}

From the practical point of view, it can be more effective to alternate the training procedure through the epochs, since within each epoch it is assumed that the model sees all the available training samples; moreover, concerning compatibility with Deep Learning frameworks, it is easier to develop a training procedure that alternatively switches the trainable weights ($\v{w}_0$ and $\v{w}_{\rm ts}$) at given epochs instead at each mini-batch.

With an alternate training through the epochs, the weights of the NN are alternatively updated for some epochs with respect to $\v{w}_0$, and for some other epochs with respect to $\v{w}_{\rm ts}$. We can outline the procedure as follows:
\begin{itemize}
    \item train the MTNN with SG for $E_0\in\N$ epochs, with respect to the shared parameters $\v{w}_0$;
    \item train the MTNN with SG for $E_{\rm ts}\in\N$ epochs, with respect to the task-specific parameters $\v{w}_{\rm ts}$;
    \item repeat until convergence (or a stopping criterion is satisfied).
\end{itemize}
So a complete cycle of alternate training (see Algorithm \ref{alg:ATE}) consists of $E_0+E_{\rm ts}$ epochs and 
$t\cdot (E_0+E_{\rm ts})$ updating iterations, where $t\in\N$ is the number of used batches per epoch, tipically the integer part of $|\mathcal{T}|/B$ for a given batch size $B$. 
We name Alternate Training trough the Epochs SG (ATE-SG) this procedure. In the rest of this section, we denote by $e$ the cycle counter and by $i$ the step (or iteration) counter, namely each cycle $e$ starts at the iterate $\v{w}^{(i_{\rm ts})}$ with $i_{\rm ts}=(E_0+E_{\rm ts})\,e\,t$.
\begin{algorithm}
\caption{ATE-SG - Alternate Training through the Epochs SG for MTNNs}
\label{alg:ATE}
\begin{description}
\item[Data:] $(\v{w}_0^{(i_{\rm ts})}, \v{w}_{\rm ts}^{(i_{\rm ts})})=\v{w}^{(i_{\rm ts})}$ (current iterate for the trainable parameters), $\mathcal{T}$ (training set),
$B$ (mini-batch size), $t$ (number of mini-batches per epoch), $E_0,E_{\rm ts}$ (number of epochs for alternate training),
$\{\eta_{i} ,~ i=i_{\rm ts} ,  \ldots, \, i_{\rm ts}+ (E_0+E_{\rm ts}) t -1\}$ (learning rates), $\ell=\ell(\mathcal{B};\v{w})$ (loss function).
 \item[Cycle $e$:]  \quad
\begin{algorithmic}[1]
\STATE $i \gets i_{\rm ts}$ ~~($\,$iteration counter, here $i=(E_0+E_{\rm ts})\,e\,t\,$)
\FOR{$e_0=1,2,\ldots, E_0$ (epochs counter, shared phase)}
    \FOR{\blue{$\tau=0,1,\ldots ,t-1$}}  
       \STATE Sample randomly and uniformly a batch $\mathcal{B}_{\tau}$ from $\mathcal{T}$ s.t.  $|\mathcal{B}_{\tau}|=B$
       \STATE $\v{w}_0 ^{(i+1)}\gets \v{w}_0^{(i)}-\eta_i\nabla_{\v{w}_0}\ell(\mathcal{B}_\tau;\v{w}^{(i)})$  ~~~
       \STATE $i \gets i+1$
       \STATE $\v{w}^{(i)} \gets  (\v{w}_0^{(i)}, \v{w}_{\rm ts}^{(i_{\rm ts})})~~$
    \ENDFOR
\ENDFOR
\STATE $i_{0} \gets i$      \quad (here $i=i_{\rm ts}+t E_0$)
\FOR{$e_{\rm ts}=1,2,\ldots, E_{\rm ts}$ (epochs counter, task-specific phase)}
    \FOR{\blue{$\tau=0,1,\ldots ,t-1$}}  
        \STATE Sample randomly and uniformly a batch $\mathcal{B}_{\tau}$ from $\mathcal{T}$ s.t. $|\mathcal{B}_{\tau}|=B$
        \STATE $\v{w}_{\rm ts}^{(i+1)} \gets \v{w}_{\rm ts}^{(i)}                 - \eta_i \nabla_{\v{w}_{\rm ts}}\ell(\mathcal{B}_\tau; \v{w}^{(i)})$
        \STATE $ i \gets i+1$
        \STATE $\v{w}^{(i)} \gets  (\v{w}_0^{(i_{0})}, \v{w}_{\rm ts}^{(i)})$
    \ENDFOR
\ENDFOR
\STATE $i_{\rm ts} \gets i$      \quad (here $i=i_0+t E_{\rm ts}$)
\RETURN $\v{w}^{(i_{\rm ts})}$ \quad (updated iterate after $E_0+ E_{\rm ts}$ epochs)
\end{algorithmic}
\end{description}
\end{algorithm}

To analyze the convergence properties of ATE-SG it is convenient to split the set of iteration indexes into two subsets, $I_0$ and $I_{\rm ts}$, corresponding to the shared phase and to the task-specific phase, respectively. In other words, gradients with respect to the shared (task-specific) parameters are evaluated at iterate $\v{w}^{(i)}$ when $i\in I_0 \,(I_{\rm ts})$.

\blue{We now prove the almost sure convergence of ATE-SG assuming to use diminishing sequences of learning rates $\{\eta_i\}_{i\in I_0}$ and $\{\eta_i\}_{i\in I_{\rm ts}}$.}

\begin{lemma}[ATE-SG]\label{teo:multi_alt_grads}
Given $ \v{w}^{(0)}\in \R^p$, let $ \{\v{w}^{(i)}\}_{i\ge 0} \subset \R^p$ be the sequence 
generated by iterating ATE-SG Algorithm \ref{alg:ATE}, 
with 
\blue{
diminishing sequences of learning rates
$\{\eta_i\}_{i\in I_0}$ and $\{\eta_i\}_{i\in I_{\rm ts}}$ such that 
\begin{equation}\label{eq:lr_hypotheses_0ts}
  a)~~~  \sum_{i\in I_0} \eta_i =  \infty\,, \quad \sum_{i\in I_{\rm ts}} \eta_i =  \infty\,, \qquad   \qquad b)~~~ 
 \sum_{i=0}^{\infty} \eta_i^2 < \infty\,,  
\end{equation}
\noindent }
where $I_0$ and $I_{\rm ts}$ are the sets of iteration indexes corresponding to the shared phase and to the task-specific phase, respectively.
Assume that the loss function $\ell$ is bounded below by $\ell_{low}$, 
  that 
$\nabla_{\v{w}_0}\ell(\,\cdot\,;\v{w})$ and  $\nabla_{\v{w}_{\rm ts}}\ell(\,\cdot\,;\v{w})$  satisfy the Lipschitz continuity
conditions (\ref{lipschitz}),
and there exist two positive constants $M_1$ and $M_2$ such that for any batch $\mathcal{B}$
 \begin{eqnarray*}
\E[\Vert\nabla_{\v{w}_0} \ell(\mathcal{B};\v{w}^{(i)})\Vert^2 | \mathcal A_i]\le M_2 \, \Vert\nabla_{\v{w}_0}\ell(\v{w}^{(i)})\Vert^2+M_1, ~~{\rm for} ~i\in I_0,\\   
\E[\Vert\nabla_{\v{w}_{\rm ts}}\ell(\mathcal{B};\v{w}^{(i)})\Vert^2 | \mathcal A_i]\le M_2 \, \Vert\nabla_{\v{w}_{\rm ts}}\ell(\v{w}^{(i)})\Vert^2+M_1, ~~{\rm for} ~i\in I_{\rm ts},
\end{eqnarray*}
where $\mathcal A_i$ denotes the 
$\sigma$-algebra induced by $ \v{w}^{(0)}$,..., $ \v{w}^{(i)}$. 
\noindent
Then 
 \begin{equation*}\label{liminf_inparti}
 \lim\inf_{i\in I_0}  ||\nabla_{\v{w}_0} \ell(\v{w}^{(i)})||^2= \lim\inf_{i\in I_{\rm ts}} ||\nabla_{\v{w}_{ts}} \ell(\v{w}^{(i)})||^2=0\quad\quad a.s.
 \end{equation*}
\end{lemma}
\begin{proof}
As in the proof of Lemma \ref{teo:simp_alt_grads}, 
we start by considering the updating 
of the shared parameters $\v{w}_0$ and the inequality: 
\begin{equation*}
    \ell(\v{w}^{(i+1)}) \le \ell(\v{w}^{(i)})
      ~+ ~ \eta_i \,\nabla \ell(\v{w}^{(i)})^T 
      \begin{bmatrix}  -\nabla_{\v{w}_0}\ell(\mathcal{B}_\tau;\v{w}^{(i)}) \\  \v{0}
     \end{bmatrix}
    +\frac{L_0}{2}\,\eta_i^2 \,\Vert \nabla_{\v{w}_0}\ell(\mathcal{B}_\tau;\v{w}^{(i)}) \Vert^2,\\
\end{equation*}
which holds for $i\in I_0$, and in conditioned expected value becomes
\begin{equation*} \label{expect0_ste} 
\begin{split} 
    \E[\ell(\v{w}^{(i+1)}) | \mathcal A_i]  
         &   ~\le ~ \ell(\v{w}^{(i)})
      -  \eta_i \,\Vert\nabla_{\v{w}_0} \ell(\v{w}^{(i)})\Vert^2
        +\frac{L_0}{2}\,\eta_i^2 \, ( M_2 \,\Vert\nabla_{\v{w}_0}\ell(\v{w}^{(i)})\Vert^2+M_1)\\
        &   ~= ~ \ell(\v{w}^{(i)})
      -  \eta_i \,G_i \,\Vert\nabla_{\v{w}_0} \ell(\v{w}^{(i)})\Vert^2
        +\frac{L_0}{2}\,\eta_i^2 \, M_1,
\end{split}
 \end{equation*}
 \blue{
 with $G_i=1-\frac{L_0}{2}\eta_i M_2\ge G_0= 1-\frac{L_0}{2}\eta_0 M_2 >0$ for sufficiently small values of $\eta_0$.}\\
 Similarly, in the task-specific case
 we have
\begin{equation*}\label{expect2_ste}
\E[\ell(\v{w}^{(i+1)}) | \mathcal A_{i}] ~\le~
 \ell(\v{w}^{(i)})
      -  \eta_i \,G_i
      \,\Vert\nabla_{\v{w}_{\rm ts}} \ell(\v{w^{(i)}})\Vert^2
        +\frac{L_{\rm ts}}{2}\,\eta_i^2 \, M_1,
\end{equation*}
\medskip
\blue{
 with $G_i=1-\frac{L_{\rm ts}}{2}\eta_i M_2\ge G_{tE_0}= 1-\frac{L_{\rm ts}}{2}\eta_{tE_0} M_2 >0$
for any $i\in I_{\rm ts}$ and sufficiently small values of the initial learning rate, $\eta_{tE_0}$,
of the task-specific updating phase. \\[1ex]
Then, setting} 
$U_i=\ell(\v{w}^{(i)})-\ell_{low}>0$ and
$\,\beta_i=0\,$ for any $i\ge 0$,
\blue{
$\,\xi_i= \frac{L_{\rm ts}}{2}\,\eta_i^2 \, M_1$ and
$\rho_i=\eta_i G_i\Vert\nabla_{\v{w}_{\rm ts}} \ell(\v{w}^{(i)})\Vert^2\,$ 
for~ $i\in I_{\rm ts}$, $\,\xi_i= \frac{L_0}{2}\,\eta_i^2 \, M_1$ and
$\,\rho_i=\eta_i G_i\Vert\nabla_{\v{w}_0} \ell(\v{w}^{(i)})\Vert^2$ \,for \, $i\in I_0$, 
since assumption (\ref{eq:lr_hypotheses_0ts}.b) ensures
$\sum_{i=0}^{\infty} \xi_i<\infty$,
by Theorem \ref{RS71}}  the sequence $\{\ell(\v{w}^{(i)})\}$ is convergent almost surely, and
\begin{equation} \label{somme_rho}
\sum_{i=0}^{\infty} \rho_i=\sum_{i\in I_{\rm ts}}  \eta_i G_i \Vert\nabla_{\v{w}_{\rm ts}} \ell(\v{w}^{(i)})\Vert^2 ~+ ~ \sum_{i\in I_{0}}  \eta_i G_i \Vert\nabla_{\v{w}_0} \ell(\v{w}^{(i)})\Vert^2<\infty \quad\quad a.s.
\end{equation}

Recall now that 
\blue{
for $\eta_0$ and $\eta_{tE_0}$ sufficiently small, 
$0<\min \{G_0,G_{tE_0}\}   \le G_i <1$ for all $i\ge 0$.}
Hence from \eqref{somme_rho} it also holds
\begin{equation} \label{senza_G_i}
 \sum_{i\in I_{\rm ts}}  \eta_i   \Vert\nabla_{\v{w}_{\rm ts}} \ell(\v{w}^{(i)})\Vert^2 ~+ ~ \sum_{i\in I_{0}}  \eta_i   \Vert\nabla_{\v{w}_0} \ell(\v{w}^{(i)})\Vert^2<\infty \quad\quad a.s.
\end{equation}
Finally, 
the thesis follows from assumption (\ref{eq:lr_hypotheses_0ts}.a).
\end{proof}

\begin{theorem}[ATE-SG convergence] \label{liminf}
Under the assumptions of Lemma \ref{teo:multi_alt_grads}, and using the same notations, let us further assume that 
there exists $M>0$ such that
\begin{equation} \label{gradient_bound}
\Vert\nabla\ell(\mathcal{B};\v{w}^{(i)})\Vert^2 \le M\,   ~~{\rm for ~any} ~\mathcal{B} ~\subset \mathcal{T},
\end{equation}
\revtwo{
and that
\begin{equation}\label{eq:lr_hypotheses_0tsmin}
      \sum_{i=0}^{\infty} \eta_i^{\rm min} = \infty\,,  
\end{equation}
where $\eta_i^{\rm min}$ denotes the minimum learning rate within one cycle of 
Algorithm \ref{alg:ATE}:
\begin{equation} \label{etamin_e} \eta_i^{\rm min} = 
\min_{j\in I_e} \{ \, \eta_j\, \}  \,~~{\rm for~} i\in I_e, 
\end{equation}
with $I_e=\{ i\in \N:  (E_0+E_{\rm ts})\,e t \,\le\, i\, \le \, 
(E_0+E_{\rm ts})\,(e+1) t -1\}$.
}
Then
 \begin{equation}\label{liminf_tutto}
 \lim\inf_{i\to \infty}  ||\nabla  \ell(\v{w}^{(i)})||^2=0\quad\quad a.s.
 \end{equation}
\end{theorem}

\begin{proof}
Similarly to Lemma \ref{teo:multi_alt_grads}, we are going to show that 
$$
\sum_{i=0}^{\infty}    \eta_i^{\rm min}   \Vert\nabla  \ell(\v{w}^{(i)} )\Vert^2<\infty \quad\quad a.s.,
$$
then the thesis follows by assumption \eqref{eq:lr_hypotheses_0tsmin}. 
To this aim, we first rewrite the above series as
\begin{equation} \label{somme} \begin{split}
\sum_{i=0}^{\infty}    \eta_i^{\rm min}  \Vert\nabla  \ell(\v{w}^{(i)} )\Vert^2 = & 
\sum_{i\in I_{\rm ts}}  \eta_i^{\rm min}  \Vert\nabla_{\v{w}_{\rm ts}} \ell(\v{w}^{(i)})\Vert^2 ~+ ~ \sum_{i\in I_{0}}  \eta_i^{\rm min}  \Vert\nabla_{\v{w}_0} \ell(\v{w}^{(i)})\Vert^2 +\\
& \sum_{i\in I_{\rm ts}}  \eta_i^{\rm min}  \Vert\nabla_{\v{w}_0} \ell(\v{w}^{(i)})\Vert^2 ~+ ~ \sum_{i\in I_0}  \eta_i^{\rm min}  \Vert\nabla_{\v{w}_{\rm ts}} \ell(\v{w}^{(i)})\Vert^2.
\end{split}
\end{equation}
\blue{Since $\eta_i^{\rm min} \le \eta_i$,} 
recalling \eqref{senza_G_i} we only need to show that the last two terms in \eqref{somme} are bounded.
Both terms can be treated in a very similar way, hence we will prove the result in detail only for the last one.
Moreover, for simplicity's sake, we consider the case $E_0=E_{\rm ts}=1$, where one cycle of ATE-SG consists of $t$ SG steps taken first with respect to the shared parameters $\v{w}_0$, and then with respect to the task-specific parameters $\v{w}_{\rm ts}$. This allows us to write down and exploit the following representations of the sets $I_0$ and $I_{\rm ts}$: 
\begin{eqnarray} \label{I0}
I_0= [\,0:t-1\,] \,\cup\, [\,2\,t:3\,t-1\,]\, \cup \,\cdots =
\displaystyle \bigcup_{e=0}^{\infty}\, [\,2\,e\,t:(2\,e+1)\,t-1\,] \quad \\
\label{Its}
\qquad  I_{\rm ts}= [\,t:2\,t-1\,] \,\cup\, [\,3\,t:4\,t-1\,]\, \cup \, \cdots =
\displaystyle \bigcup_{e=0}^{\infty}\, [\,(2\,e+1)\,t:2(e+1)t-1\,].
\end{eqnarray}
Nonetheless, we remark that with a bit more technicalities the proof can be extended to the general case $E_0\ge 1$ and $E_{\rm ts}\ge 1$. 

So, let us consider the last term in \eqref{somme}, which by using \eqref{I0} 
can be rewritten as follows:
\begin{equation} \begin{split} \label{sharing}
    \sum_{i\in I_0}  \eta_i^{\rm min}   \Vert\nabla_{\v{w}_{\rm ts}} \ell(\v{w}^{(i)})\Vert^2 
    & = \sum_{e=0}^{\infty} \sum_{\tau=0}^{t-1}  \eta_{2 e t+\tau}^{\rm min}   \Vert\nabla_{\v{w}_{\rm ts}} \ell(\v{w}^{(2 e t+\tau)})\Vert^2\\
    & = \sum_{\tau=0}^{t-1}  \eta_{\tau}^{\rm min}  \Vert\nabla_{\v{w}_{\rm ts}}   \ell(\v{w}^{(\tau)})\Vert^2 
    + \sum_{e=1}^{\infty} \sum_{\tau=0}^{t-1}  \eta_{2 e t+\tau}^{\rm min}   \Vert\nabla_{\v{w}_{\rm ts}} \ell(\v{w}^{(2 e t+\tau)})\Vert^2,
    \end{split}
\end{equation}
where, for any $e\ge 1$ and $\tau=0,1,\ldots,t-1$,
\begin{equation*} \begin{split}
\Vert\nabla_{\v{w}_{\rm ts}} \ell(\v{w}^{(2 e t+\tau)})\Vert & 
\le \Vert\nabla_{\v{w}_{\rm ts}} \ell(\v{w}^{(2 e t+\tau)})-\nabla_{\v{w}_{\rm ts}} \ell(\v{w}^{(2 e t -1)})\Vert +
\Vert\nabla_{\v{w}_{\rm ts}} \ell(\v{w}^{(2 e t -1)})\Vert\\
& \le L_{\rm ts}^0 \Vert \v{w}^{(2 e t+\tau)}-\v{w}^{(2 e t -1)}\Vert +
\Vert\nabla_{\v{w}_{\rm ts}} \ell(\v{w}^{(2 e t -1)})\Vert\\
& = L_{\rm ts}^0   \, \Vert  
\eta_{2et-1} \nabla_{\v{w}_{\rm{ts}}}
\ell(\mathcal{B}_{\blue{t-1}};\v{w}^{(2 e t-1)})    +
\sum_{k=0}^{\tau-1} \eta_{2et+k} \nabla_{\v{w}_{0}}
\ell(\mathcal{B}_k;\v{w}^{(2 e t+k)}) 
\, \Vert     \\
 & \quad + \Vert\nabla_{\v{w}_{\rm ts}} \ell(\v{w}^{(2 e t -1)})\Vert.
\end{split}
\end{equation*}
The idea of the proof is that, when damped by the learning rates,
gradient norms can be bounded by assumption \eqref{gradient_bound}.
Instead, to tackle the norm of 
$\nabla_{\v{w}_{\rm ts}} \ell(\v{w}^{(2 e t -1)})$ we can resort
to Lemma \ref{teo:multi_alt_grads}, because 
by \eqref{Its} it is easily seen that
$2et-1 \in I_{\rm ts}$ for any $e\ge 1$. Then
$$
~\Vert\nabla_{\v{w}_{\rm ts}} \ell(\v{w}^{(2 e t+\tau)})\Vert
\le L_{\rm ts}^0  M \sum_{k=-1}^{\tau-1} \eta_{2et+k} +
\Vert\nabla_{\v{w}_{\rm ts}} \ell(\v{w}^{(2 e t-1)})\Vert\,,  
$$ 
and
\begin{equation} \begin{split} \label{norma_q}
\Vert\nabla_{\v{w}_{\rm ts}} \ell(\v{w}^{(2 e t+\tau)})\Vert^2  
& \le (L_{\rm ts}^0 )^2 M^2 \,(\,\sum_{k=-1}^{\tau-1} \eta_{2et+k}\,)^2 +\\
 & \quad 2 L_{\rm ts}^0  M^2 \sum_{k=-1}^{\tau-1} \eta_{2et+k}\,  + 
 \Vert\nabla_{\v{w}_{\rm ts}} \ell(\v{w}^{(2 e t -1)})\Vert^2 .
\end{split} \end{equation}
Now using \eqref{gradient_bound} and \eqref{norma_q} in \eqref{sharing} we have
\begin{equation} \begin{split} \label{diminishing}
    \sum_{i\in I_0}  \eta_i^{\rm min}  \Vert\nabla_{\v{w}_{\rm ts}} \ell(\v{w}^{(i)})\Vert^2 & \le M^2 \sum_{\tau=0}^{t-1} \eta_{\tau}    +
    (L_{\rm ts}^0)^2 M^2 \sum_{e=1}^{\infty} \sum_{\tau=0}^{t-1}  \eta_{2 e t+\tau}^{\rm min} 
    (\,\sum_{k=-1}^{\tau-1} \eta_{2et+k}\,)^2  \\
 & + 2  L_{\rm ts}^0  M^2 \sum_{e=1}^{\infty} \sum_{\tau=0}^{t-1}  \eta_{2 e t+\tau}^{\rm min} 
    (\sum_{k=-1}^{\tau-1} \eta_{2et+k})  \\
  &  +\sum_{e=1}^{\infty} \sum_{\tau=0}^{t-1}  \eta_{2 e t+\tau}^{\rm min}  \Vert\nabla_{\v{w}_{\rm ts}} \ell(\v{w}^{(2 e t -1)})\Vert^2.
    \end{split}
\end{equation}
\blue{It is worth to point out here a technical fact that we exploit to conclude the proof: 
by \eqref{I0} and   \eqref{Its} it is easily seen that $\{2et, 2et+1,\ldots,2et+t-1\}\subset I_0$ 
 for all $e\ge 0$, and $\{2et-1, 2et+t, 2et+t+1,\ldots,2et+2t-1\}\subset I_{\rm ts}$ 
 for all $e\ge 1$. The corresponding sequences of learning rates are diminishing, so that 
 recalling definition \eqref{etamin_e} we have
\begin{equation}\label{etamin-ineq}
\eta_{2 e t+\tau}^{\rm min}= \min \{ \eta_{2 e t+t-1} , \eta_{2 e t+2t-1} \} <  
  \eta_{2 e t - 1},~~  \forall~e\ge 1 ~~{\rm and}~ 0 \le \tau \le t-1.
\end{equation}  
}
Hence the last term in \eqref{diminishing} can be bounded almost surely by exploiting
\eqref{senza_G_i} as follows: 
\begin{equation*} \begin{split}
\sum_{e=1}^{\infty} \sum_{\tau=0}^{t-1}  \eta_{2 e t+\tau}^{\rm min}  \Vert\nabla_{\v{w}_{\rm ts}} \ell(\v{w}^{(2 e t -1)})\Vert^2 
&~<~ 
t\sum_{e=1}^{\infty}  \eta_{2 e t-1} 
\Vert\nabla_{\v{w}_{\rm ts}} \ell(\v{w}^{(2 e t -1)})\Vert^2 \\ &~ <~
t  \sum_{i\in I_{\rm ts}}  \eta_i   
\Vert\nabla_{\v{w}_{\rm ts}} \ell(\v{w}^{(i)})\Vert^2  <\infty.
\end{split} \end{equation*}
The other terms can be bounded using 
\blue{the assumption 
 $\sum_{i=0}^{\infty} \eta_i^2 < \infty$. 
Indeed,  besides \eqref{etamin-ineq},
since $\{2et, 2et+1,\ldots,2et+t-1\}\subset I_0$ and the corresponding sequence of learning rate is diminishing, 
it  holds also that
$$\eta_{2 e t+\tau}^{\rm min}\le \eta_{2 e t + \tau}<\eta_{2 e t+k} 
\le \eta_{2 e t} , ~~{\rm  for~ all~~} e\ge 1,~~ 
1\le \tau \le t-1 ~~{\rm and~~}  0\le k \le \tau-1; $$ 
then we have
\begin{equation*} \begin{split}
 \sum_{e=1}^{\infty} \sum_{\tau=0}^{t-1}  \eta_{2 e t+\tau}^{\rm min}  \, 
 (\sum_{k=-1}^{\tau-1} \eta_{2 e t+k}\,) & = 
 \sum_{e=1}^{\infty}~ ( ~\sum_{\tau=0}^{t-1}  \eta_{2 e t+\tau}^{\rm min}  \, 
 \eta_{2 e t-1} + \sum_{\tau=1}^{t-1} \sum_{k=0}^{\tau-1} \eta_{2 e t+\tau}^{\rm min} \eta_{2 e t+k}\,) \\
 & <  \sum_{e=1}^{\infty} ~ ( ~\sum_{\tau=0}^{t-1}   \, 
 \eta_{2 e t-1}^2 + \sum_{\tau=1}^{t-1}\sum_{k=0}^{\tau-1} \eta_{2 e t}^2\,) 
    \\
 & < t\, \sum_{e=1}^{\infty}   \eta_{2 e t-1}^2 
 + (t-1)^2\,\sum_{e=1}^{\infty} \eta_{2 e t}^2 < \infty,
 \end{split} \end{equation*}
 
and
\begin{equation*}  
 \sum_{e=1}^{\infty} \sum_{\tau=0}^{t-1}  \eta_{2 e t+\tau}^{\rm min} \, 
 (\sum_{k=-1}^{\tau-1} \eta_{2 e t+k}\,)^2 
   <  (\sum_{k=-1}^{\tau-1} \eta_{2   t+k}\,) \,\sum_{e=1}^{\infty} \sum_{\tau=0}^{t-1}  \eta_{2 e t+\tau}^{\rm min} \, 
 (\sum_{k=-1}^{\tau-1} \eta_{2 e t+k}\,) 
 < \infty.
  \end{equation*}
  }
\end{proof}

 

\begin{remark} It is worth to observe that assumption \eqref{gradient_bound} 
in Theorem \ref{liminf} can be relaxed.
\blue{
Indeed, with a slight modification of the theorem's proof, we can prove that 
the last  term in \eqref{somme} is bounded assuming 
\begin{equation*}
 \sum_{e=1}^{\infty}    \eta_{2 e t-1}^2 M_{2et-1}^2<\infty\quad\quad
 {\rm and} \quad \quad
  \sum_{e=1}^{\infty}    \eta_{2 e t}^2 M_{2et-1}^2<\infty,
\end{equation*}
where
\begin{equation*}
M_{2 e t-1}=\max\{\Vert \nabla_{\v{w}_{\rm{ts}}}
\ell(\mathcal{B}_{t-1};\v{w}^{(2 e t-1)}) \Vert, \max_{k=0,1,\ldots,t-1} \Vert \nabla_{\v{w}_{0}}
\ell(\mathcal{B}_k;\v{w}^{(2 e t+k)})\Vert \}.
\end{equation*}
These conditions are clearly weaker than \eqref{gradient_bound},
and can be satisfied for example whenever $\{\eta_i\}=\{\frac{1}{i}\}$,
with $p\in (0,1/2)$.}
\end{remark}

\section{Numerical Experiments}\label{sec:num_exp}
In this section, we report the results of two numerical experiments. Both experiments study
and compare the performance of an MTNN trained with a standard stochastic gradient procedure and by
using alternate training procedures. The first experiment takes into account a synthetic dataset 
representing a 4-classes classification task and a binary classification task for a set of points 
in $\R^2$; 
the second experiment is based on a real-world dataset, where signals must be 
classified with respect to two different multiclass classification tasks.
In both cases, the used
loss function is a weighted sum of the task-specific losses where the weights are fixed and equal
to one (i.e., $\lambda_1=\lambda_2=1$, see \eqref{eq:total_loss}). Indeed, we want to analyze the 
effects of our method without any balancing of the task-specific losses or emphasizing the action
of a task with respect to the other.

The alternate procedure used for the experiments of this section is illustrated in \Cref{alg:practical_ATE} of \Cref{app:implATESG}: we call it \emph{implemented} ATE-SG method. 
The main difference between this implemented version and ATE-SG (see \Cref{alg:ATE}) lies in the way mini-batches are generated. Indeed, for theoretical purposes, in ATE-SG the mini-batches are randomly sampled from the training set $\mathcal{T}$, instead of being obtained through a random split of $\mathcal{T}$ as in \Cref{alg:practical_ATE}. This modification is necessary for optimal compatibility with the Deep Learning frameworks {\it Keras} \cite{keras2015} and {\it TensorFlow} \cite{tensorflow2015-whitepaper} used in the experiments.
\blue{Starting from a given initial value, the learning rates are adaptively reduced as specified in the training options below. }

\blue{
To compare the performance of the implemented ATE-SG procedures with the standard SG procedure, we
investigated their dependence both on the initial learning rate and on the number of alternate update epochs $E_0=E_{\rm ts}$. In addition, each test case was repeated several times using different random seeds. 
Specifically,
we trained the MTNN   for each combination of the following hyperparameters:
\begin{itemize}
    \item 11 distinct random seeds;
    \item Starting learning rates: $10^{-2}$, $10^{-3}$;
    \item Alternate training sub-epochs: $E_0=E_{\rm ts}=1,10,100$. 
\end{itemize}   
\noindent
Training options common to both experiments are the following:
\begin{itemize}
    \item Preprocessing: standard scaler for inputs (based on training data only);
    \item Optimization: maximum number of epochs $5000$;
  early stopping (patience 350, restore best weights); 
    reduce learning rate on plateaus (patience 50, factor 0.75, min-delta $0.0001$).
\end{itemize}
}

\blue{
For classification tasks,  typical performance measures of the methods are
 the weighted average precision, 
 the weighted average recall, 
and the weighted average $F_1$-score \cite{sklearn}.
For the reader's convenience, we report a brief description of the quantities: precision, recall, and $F_1$-score, for any class $C$ of a general classification problem (see \cite{Makhoul1999_precrecf1}). The precision is the percentage of correct predictions among all the elements predicted as $C$; the recall is the percentage of elements predicted as $C$ among all the $C$ elements in the set; the $F_1$-score is defined as the weighted harmonic mean of precision and recall. Then, the weighted average precision/recall/$F_1$-score is the weighted average of these quantities with respect to the cardinalities of each class in the set. For this reason, we have that the weighted average recall is equivalent to the accuracy.

\noindent
Performance measures reported in 
Tables \ref{tab:synth_performance_1e-2}-\ref{tab:synth_performance_1e-3} 
and Tables \ref{tab:rco_performance_1e-2}-\ref{tab:rco_performance_1e-3}
are averaged with respect to the 11 models trained with different random seeds. 
We also analyze the behavior of the  loss functions 
(on both  the training and the validation set) over the epochs and plot it in Figures \ref{fig:synth_losses_1e-2}-\ref{fig:synth_losses_1e-3} and Figures
\ref{fig:rco_losses_1e-2}-\ref{fig:rco_losses_1e-3}: in the figures,
colored areas represent the values between the minimum and the maximum loss obtained 
after 11 realizations; colored lines are the median values.

}


\subsection{Experiments on Synthetic Data}\label{sec:exp_synthetic}

We consider a dataset $\mathcal{D}$ made of $N=10\, 000$ points uniformly sampled in the square $D=[-2, 2]^2\subset\R^2$ and labeled with respect to two different criteria: $1$) to belong to one of the four quadrants of $\R^2$ (4-classes classification task); $2$) to be inside or outside the unitary circle centered in the origin (binary classification task). For simplicity, we denote these tasks by \emph{task 1} and \emph{task 2}, respectively.

Then, the dataset $\mathcal{D}$ is made up of samples $(\v{x}_i, (q_i, c_i))$ such that $\v{x}_i\in D=[-2, 2]^2$, $q_i \in \{0, \ldots ,3\}$, and $c_i\in\{0, 1\}$, for each $i=1,\ldots , N$ (see \Cref{fig:sinth_data}), where $q_i$ and $c_i$ denote the quadrant label and the circle label of $\v{x}_i$, respectively.

\begin{figure}[htb!]
    \centering
    \includegraphics[width=0.65\textwidth]{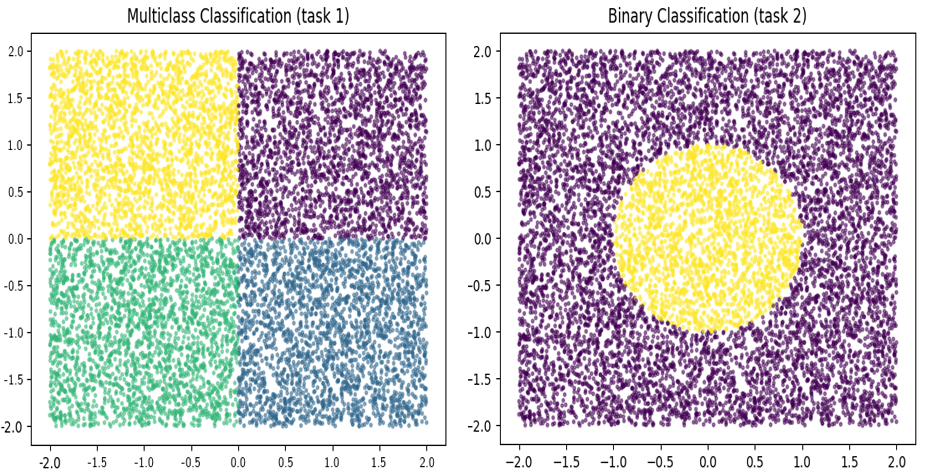}
    \caption{The synthetic dataset $\mathcal{D}$ with respect to its two tasks. Task 1 on the left, task 2 on the right. Different colors denote different labels for the points.
    }
    \label{fig:sinth_data}
\end{figure}

We randomly split $\mathcal{D}$, so that the training set $\mathcal{T}$, the validation set $\mathcal{V}$, and the test set $\mathcal{P}$ are made up of $T=5\,600$ samples, $V=1\,400$ samples, and $P=3\,000$ samples, respectively. Then, we build an MTNN with architecture as described in \Cref{tab:mtnn_arch_synth} and, 
both for the classic SG and the implemented ATE-SG procedures, we train it 
\blue{ by setting the mini-batch size to 256 and minimizing
the sum of the following task-specific loss functions:
\begin{itemize}
    \item task 1: categorical cross-entropy;
    \item task 2: binary cross-entropy (from logits). 
\end{itemize}}


\begin{table}[htb!]
    \centering
    \resizebox{1.\textwidth}{!}{
    \begin{tabular}{|c|c||c|c|c|c|}
        \hline
        Layer Name & Layer Type (Keras) & Output Shape & Param. \# & Param. \# (Grouped) & Connected to\\
        \hline
        \hline
        input & InputLayer & (None, 2) & $0$ & - & - \\
        \hline
        trunk\_01 & Dense (relu) & (None, 512) & $1\,536$ & ($\v{w}_0$-param.s) & input\\
        trunk\_02 & Dense (relu) & (None, 512) & $262\,656$ & $526\,848$ & trunk\_01\\
        trunk\_03 & Dense (relu) & (None, 512) & $262\,656$ & \quad & trunk\_02\\
        \hline
        quad\_01 & Dense (relu) & (None, 512) & $262\,656$ & ($\v{w}_{\rm ts}$-param.s, task 1) & trunk\_03 \\
        quad\_02 & Dense (relu) & (None, 512) & $262\,656$ & $527\,364$ & quad\_01 \\
        quad\_out & Dense (softmax) & (None, 4) & $2\,052$ & \quad & quad\_02\\
        \hline
        circ\_01 & Dense (relu) & (None, 512) & $262\,656$ & ($\v{w}_{\rm ts}$-param.s, task 2) & trunk\_03\\
        circ\_02 & Dense (relu) & (None, 512) & $262\,656$ & $525\,825$ & circ\_01 \\
        circ\_out & Dense (linear)& (None, 1) & $513$ & \quad & circ\_02\\
        \hline
    \end{tabular}
    }
    \caption{Synthetic dataset cases. Keras \cite{keras2015,tensorflow2015-whitepaper} architecture of the MTNN.}
    \label{tab:mtnn_arch_synth}
\end{table}

\blue{After training the model with respect to all the configurations,
the first thing we observe is that the performance measured in terms of recall, precision, and $F_1$-score is almost identical for all the training methods and all the starting learning rate values. See Tables \ref{tab:synth_performance_1e-2}-\ref{tab:synth_performance_1e-3} for the average performance measures of the MTNN on the test set.
}

\begin{table}[htb!]
    \centering
    \resizebox{1.\textwidth}{!}{
    \begin{tabular}{|c||c|c||c|c||c|c|}
        \hline
        \quad & \multicolumn{2}{c||}{Recall (Accuracy)} & \multicolumn{2}{c||}{Precision} & \multicolumn{2}{c|}{$F_1$-score} \\
        \hline
        Training Type & task1 & task 2 & task1 & task 2 & task1 & task 2 \\
        \hline 
        \hline
        ATE-SG ($E_0=E_{\rm ts}=1$) & 0.997909 & 0.996667 & 0.997914 & 0.996677 & 0.997909 & 0.996665\\
        \hline
        ATE-SG ($E_0=E_{\rm ts}=10$) & 0.997727 & 0.996758 & 0.997730 & 0.996766 & 0.997727 & 0.996758 \\
        \hline
        ATE-SG ($E_0=E_{\rm ts}=100$) & 0.997848 & 0.996576 & 0.997851 & 0.996587 & 0.997848 & 0.996575\\
        \hline
        \hline
        Classic SG & 0.997788 & 0.996697 & 0.997791 & 0.996703 & 0.997788 & 0.996694 \\
        \hline
    \end{tabular}
    }
    \caption{\blue{Synthetic dataset. Starting learning rate $10^{-2}$. Performance measures of the MTNNs on the test set $\mathcal{P}$. 
    }}
    \label{tab:synth_performance_1e-2}
\end{table}

\begin{table}[htb!]
    \centering
    \resizebox{1.\textwidth}{!}{
    \begin{tabular}{|c||c|c||c|c||c|c|}
        \hline
        \quad & \multicolumn{2}{c||}{Recall (Accuracy)} & \multicolumn{2}{c||}{Precision} & \multicolumn{2}{c|}{$F_1$-score} \\
        \hline
        Training Type & task1 & task 2 & task1 & task 2 & task1 & task 2 \\
        \hline
        \hline
        ATE-SG ($E_0=E_{\rm ts}=1$) & 0.998061 & 0.996606 & 0.998063 & 0.996614 & 0.998060 & 0.996605\\
        \hline
        ATE-SG ($E_0=E_{\rm ts}=10$) & 0.998030 & 0.996606 & 0.998033 & 0.996616 & 0.998030 & 0.996606 \\
        \hline
        ATE-SG ($E_0=E_{\rm ts}=100$) & 0.998333 & 0.996818 & 0.998335 & 0.996826 & 0.998333 & 0.996817\\
        \hline
        \hline
        Classic SG & 0.998242 & 0.996758 & 0.998244 & 0.996767 & 0.998242 & 0.996757 \\
        \hline
    \end{tabular}
    }
    \caption{\blue{Synthetic dataset. Starting learning rate $10^{-3}$. Performance measures of the MTNNs on the test set $\mathcal{P}$. 
    }}
    \label{tab:synth_performance_1e-3}
\end{table}

\blue{
 We also analyze the behavior of the loss functions through the epochs:
 in Figures \ref{fig:synth_losses_1e-2} and \ref{fig:synth_losses_1e-3}, we plot the values of training and validation loss versus epochs, for $\eta=10^{-2}$ and $\eta=10^{-3}$ respectively.
 As already noticed, 
 at an equal number of epochs
 ATE-SG is less expensive and requires less memory than SG, because SG computes and stores the full gradient $\nabla_{\v{w}}\ell$, while ATE-SG computes and stores either $\nabla_{\v{w}_{0}}\ell$ or $\nabla_{\v{w}_{\rm ts}}\ell$. 
 In our test,  the dimension of $\nabla_{\v{w}_0}\ell$ is approximately one-third of 
$\nabla_{\v{w}}\ell$'s dimension (see \Cref{tab:mtnn_arch_synth}). Then, when computing $\nabla_{\v{w}_{\rm ts}}\ell$  in the task-specific phase,  ATE-SG saves one-third of memory allocation and one-third of back-propagation operations, because 
  the back-propagation \cite{Rumelhart1986_BACKPROP_Nature} stops before propagating through the trunk $\widehat{\v{F}}_0$ of the MTNN (see \Cref{def:MTNN}). 
In the shared phase, two-thirds of memory allocation is saved.

Given this premise, we observe what follows.

\begin{itemize}

    \item For the case $E_0=E_{\rm ts}=1$, ATE-SG appears to have regularization effects on the training, leading to a marked reduction in its oscillations with respect to SG as well as to larger values of $E_0$ and $E_{\rm ts}$, both in the training set and in the validation set. 

    

    
    \item For the case $E_0=E_{\rm ts}=10$, we observe a particular phenomenon for ATE-SG: the training
    yields oscillations for the loss function which show spikes approximately every $10$ epochs 
    and are characterized by a range of values larger than the typical fluctuations of the loss of a classic MTNN training. 
    Nonetheless, the overall trend of the loss on the validation set is still decreasing with the epochs (see Figures \ref{fig:synth_losses_1e-2}-(b) and \ref{fig:synth_losses_1e-3}-(b)). 
    Analogously to the other cases, these oscillations tend to reduce when decreasing the learning rate (e.g., compare Figures \ref{fig:synth_losses_1e-2}-(b) and \ref{fig:synth_losses_1e-3}-(b)).

    \item The larger the value chosen for $E_0=E_{\rm ts}$, the more similar the loss behavior of ATE-SG is to the loss of SG. In particular, with $E_0=E_{\rm ts}=100$ the loss behaviors 
        (on both training and validation set) are very similar, both in the range of values and in the stochastic oscillations 
    (see Figures  \ref{fig:synth_losses_1e-2} and \ref{fig:synth_losses_1e-3}). 
    In general, these oscillations tend to reduce with smaller learning rates due to the shorter steps in the domain.
    
    \item In Figure \ref{fig:synth_lr_decay} we plot the used learning rate against the epochs for each seed and for each method, for the starting learning rate $10^{-3}$. We can observe that SG and ATE-SG with $E_0=E_{\rm ts}=100$ shows a similar behaviour, while the initial learning rate is used for a larger number of epochs in  ATE-SG with $E_0=E_{\rm ts}=1,10$, resulting in overall larger sequences of learning rates. We observed a similar behavior also with the starting learning rate equal to $10^{-2}$. 
\end{itemize}
}

\begin{figure}[htb!]
    \centering
    \subcaptionbox{Training loss, starting learning rate $10^{-2}$.}{
    \includegraphics[trim={1.75cm 0.25cm 2.45cm 1.8cm},clip,width=.65\textwidth]{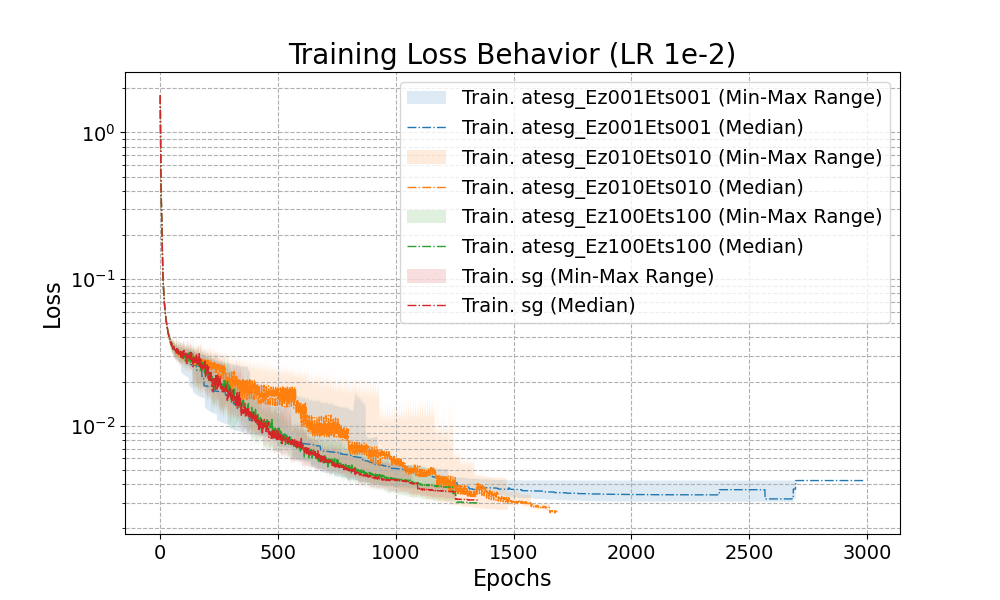}
    }
    \subcaptionbox{Validation loss, starting learning rate $10^{-2}$.}{
    \includegraphics[trim={1.75cm 0.25cm 2.45cm 1.75cm},clip,width=.65\textwidth]{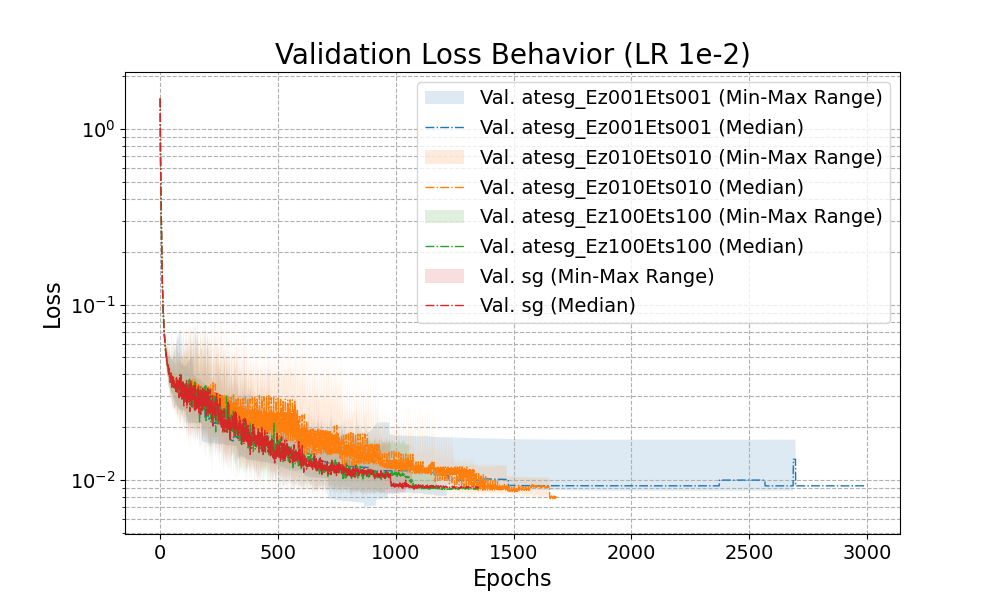}
    }
    \caption{\blue{Synthetic dataset case.} 
    }
    \label{fig:synth_losses_1e-2}
\end{figure}

\begin{figure}[htb!]
    \centering
    \subcaptionbox{Training loss, starting learning rate $10^{-3}$.}{
    \includegraphics[trim={1.75cm 0.25cm 2.45cm 1.79cm},clip,width=.65\textwidth]{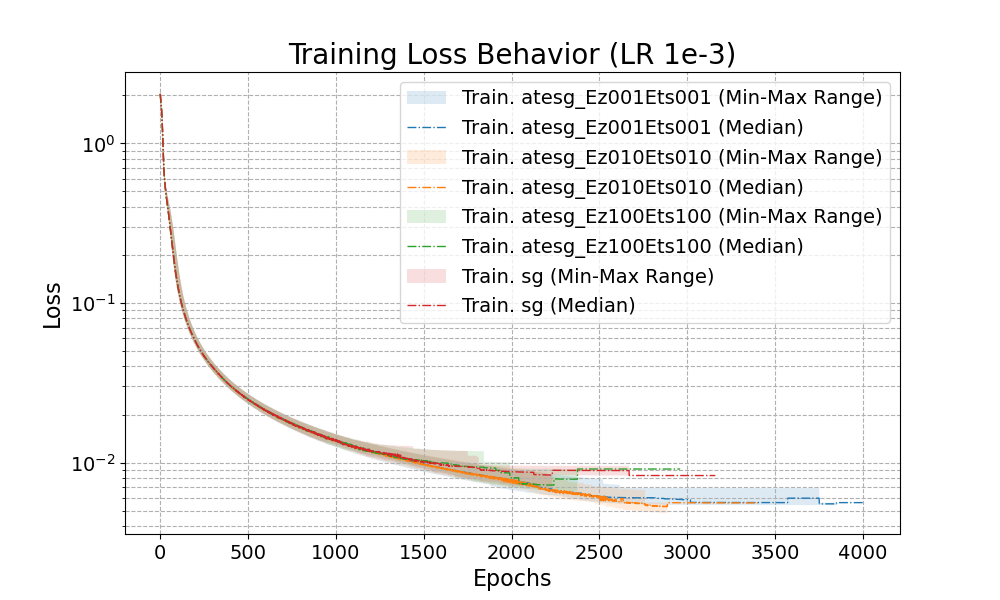}
    }
    \subcaptionbox{Validation loss, starting learning rate $10^{-3}$.}{
    \includegraphics[trim={1.75cm 0.25cm 2.45cm 1.75cm},clip,width=.65\textwidth]{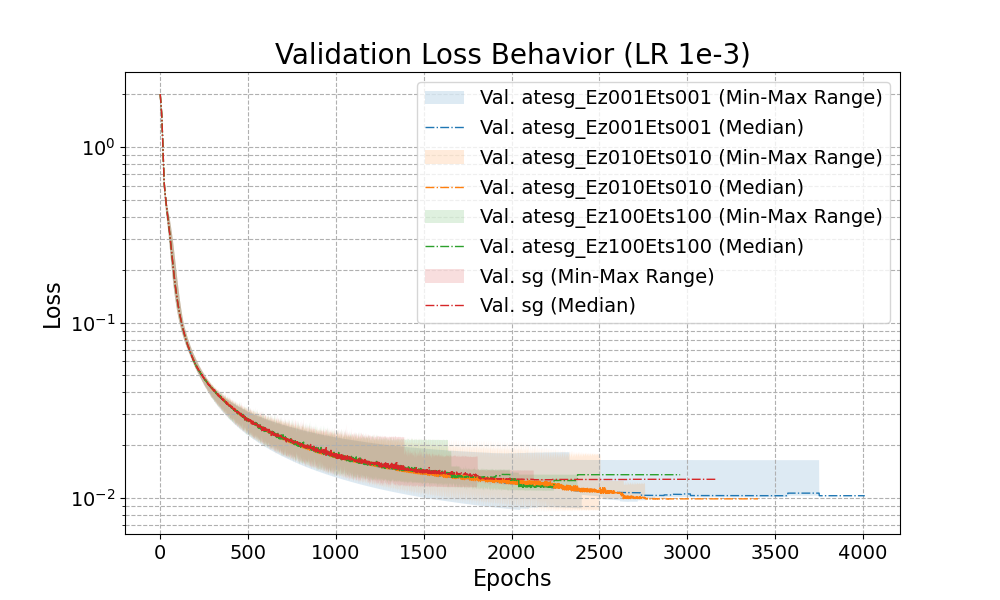}
    }
    \caption{\blue{Synthetic dataset case.} 
    }
    \label{fig:synth_losses_1e-3}
\end{figure}


\begin{figure}[htb!]
    \centering
    \includegraphics[trim={1.75cm 0.25cm 2.45cm 1.cm},clip,width=.75\textwidth]{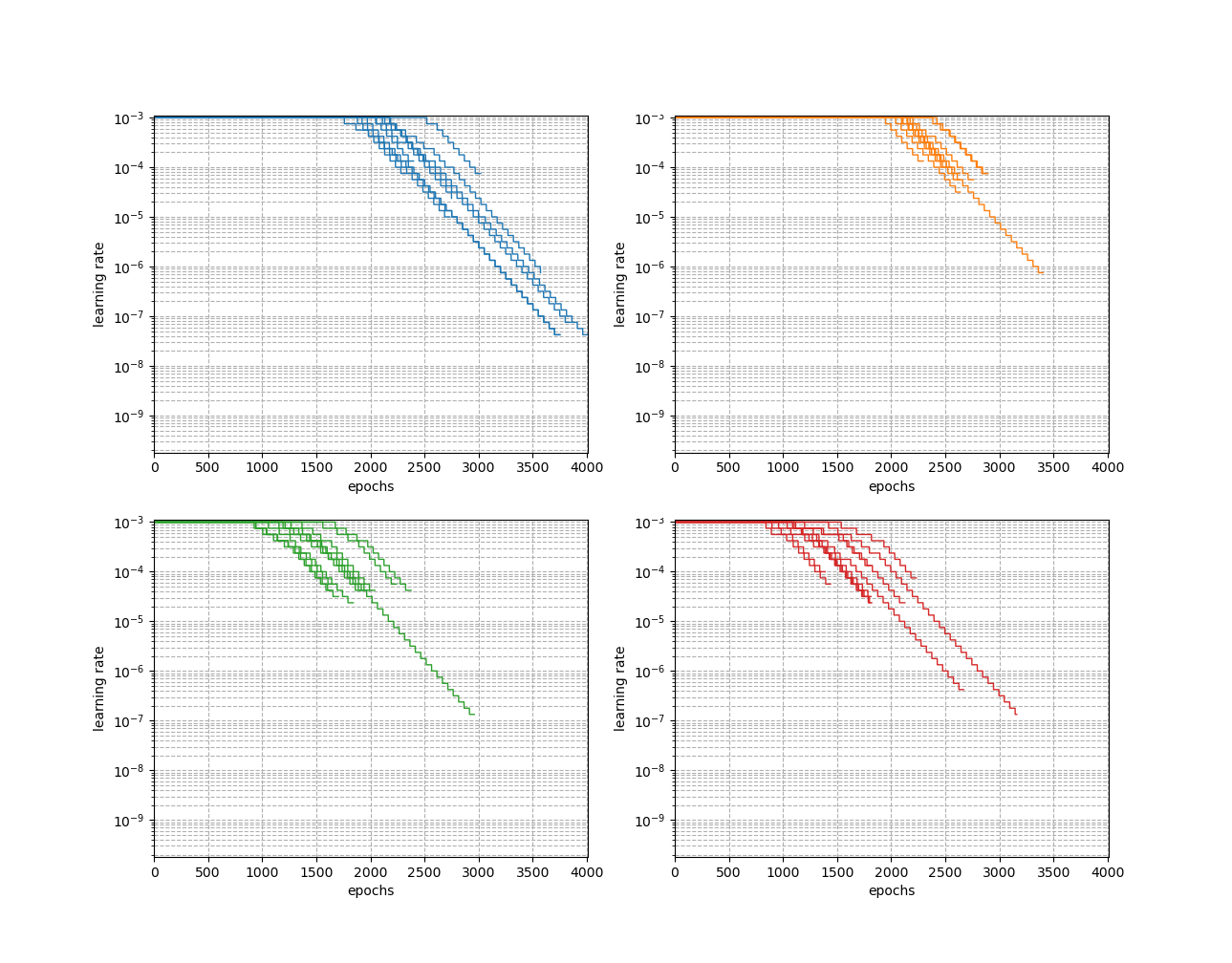}
    \caption{
    \blue{Synthetic dataset case. Learning rate sequence, starting value: $10^{-3}$. From left to right, from top to bottom: ATE-SG $E_0=E_{\rm ts}=1$, ATE-SG $E_0=E_{\rm ts}=10$, ATE-SG $E_0=E_{\rm ts}=100$, and SG.}
    }
    \label{fig:synth_lr_decay}
\end{figure}

\subsection{Experiments on Real-World Data}\label{sec:exp_signals}

For the experiment on real-world data, we report the results obtained on a dataset used for wireless signal recognition tasks \cite{Jagannath2021_dataset}. Typically, signal recognition is segmented into sub-tasks like the modulation recognition or the wireless technology (i.e., signal type); nonetheless, in \cite{Jagannath2021_MTL, Jagannath2021_arXiv} the authors suggest an approach to the problem that exploits a multi-task setting for classifying at the same time both the modulation and the signal type of a wireless signal.

Here we focus on signals characterized by $0{\rm dB}$ Signal-to-Noise Ratio (SNR). The dataset $\mathcal{D}$ is made of $N=63\, 000$ signals, each one represented as a vector $\v{s}_i\in\R^{256}$ obtained from 128 complex samples of the original signal. In the dataset, associated to each signal $\v{s}_i$, we have a label $\mu_i\in\{0,\ldots ,5\}$ for the corresponding modulation (\emph{task 1}) and a label $\sigma_i\in\{0,\ldots ,7\}$ for the corresponding signal type (\emph{task 2}). For more details about the data see \cite{Jagannath2021_dataset, Jagannath2021_MTL,Jagannath2021_arXiv}.

We split $\mathcal{D}$ randomly, so that the training set $\mathcal{T}$, the validation set $\mathcal{V}$, and the test set $\mathcal{P}$ are made of $T=35\,280$ samples, $V=8\,820$ samples, and $P=18\,900$ samples, respectively (i.e., same ratios used for the synthetic data in \Cref{sec:exp_synthetic}). Then, we build an MTNN with architecture as described in \Cref{tab:mtnn_arch_realworld} and, both for the classic SG and the implemented ATE-SG procedures, we train it \blue{by setting the mini-batch size to 512 and minimizing the sum of the task-specific loss functions (categorical cross-entropy for both tasks).}
\blue{In particular, given this architecture (see \Cref{tab:mtnn_arch_realworld} and \Cref{rem:mtnn_arch_signals} below) when computing $\nabla_{\v{w}_{\rm ts}}\ell$ in the task-specific phase, ATE-SG saves one-half of memory allocation and one-half of back-propagation operations; on the other hand, in the shared phase, one-half of memory allocation is saved.}


\begin{remark}[MTNN architecture and hyper-parameters]\label{rem:mtnn_arch_signals}
    Since the aim of the experiment is to study how the training performance of an MTNN change when using an alternate training procedure, we do not focus on hyper-parameter and/or architecture tuning for obtaining the best predictions. Then, for simplicity, in this experiment we choose a simple but efficient architecture (see \Cref{tab:mtnn_arch_realworld}) based on 1-dimensional convolutional layers, after a brief, manual hyper-parameter tuning. The 1-dimensional convolutional layers are useful to exploit the signals reshaped as 128 complex signals (see layer \emph{trunk\_00} in \Cref{tab:mtnn_arch_realworld}), keeping the NN relatively small.
\end{remark}

\begin{table}[htb!]
    \centering
    \resizebox{1.\textwidth}{!}{
    \begin{tabular}{|c|c||c|c|c|c|}
        \hline
        Layer Name & Layer Type (Keras) & Output Shape & Param. \# & Param. \# (Grouped) & Connected to\\
        \hline
        \hline
        input & InputLayer & (None, 256) & $0$ & - & - \\
        \hline
        trunk\_00 & Reshape & (None, 128, 2) & $0$ & \quad & input\\
        trunk\_01 & Conv1D (relu) & (None, 128, 64) & $576$ & ($\v{w}_0$-param.s) & trunk\_00\\
        trunk\_02 & Conv1D (relu) & (None, 128, 32) & $8\,224$ & $12\,928$ & trunk\_01\\
        trunk\_03 & Conv1D (relu) & (None, 128, 32) & $4\,128$ & \quad & trunk\_02\\
        trunk\_end & GlobalMaxPooling1D & (None, 32) & $0$ & \quad & trunk\_03\\
        \hline
        mod\_01 & Dense (relu) & (None, 64) & $2\,112$ & ($\v{w}_{\rm ts}$-param.s, task 1) & trunk\_end \\
        mod\_02 & Dense (relu) & (None, 64) & $4\,160$ & $6\,662$ & mod\_01 \\
        mod\_out & Dense (softmax) & (None, 6) & $390$ & \quad & mod\_02\\
        \hline
        sig\_01 & Dense (relu) & (None, 64) & $2\,112$ & ($\v{w}_{\rm ts}$-param.s, task 2) & trunk\_end\\
        sig\_02 & Dense (relu) & (None, 64) & $4\,160$ & $6\,792$ & sig\_01 \\
        sig\_out & Dense (softmax) & (None, 8) & $520$ & \quad & sig\_02\\
        \hline
    \end{tabular}
    }
    \caption{Real-world dataset. Keras \cite{keras2015,tensorflow2015-whitepaper} architecture of the MTNN.}
    \label{tab:mtnn_arch_realworld}
\end{table}

\blue{

Looking at Tables \ref{tab:rco_performance_1e-2}-\ref{tab:rco_performance_1e-3}, all the trained models show approximately the same performance, but 
from Figures \ref{fig:rco_losses_1e-2}-\ref{fig:rco_losses_1e-3} we can see  that experiments on this test case are more challenging than the previous ones. 
Although achieving good accuracy, with the initial learning rate $10^{-2}$ overfitting is reached; starting with learning rate $10^{-3}$ overfitting is no longer a problem, rather a plateau appears to be reached, and the training procedures have difficulties in further decreasing the validation loss.

Further, some general trends are confirmed:
\begin{itemize}
\item the choice $E_0=E_{\rm ts}=1$ yields a regularization effect on the training; 
\item oscillations tend to reduce with smaller learning rates;
\item for larger values of $E_0$ and $E_{\rm ts}$   the loss behavior of ATE-SG is closer to that of SG.
\end{itemize}

    

}

\begin{table}[htb!]
    \centering
    \resizebox{1.\textwidth}{!}{
    \begin{tabular}{|c||c|c||c|c||c|c|}
        \hline
        \quad & \multicolumn{2}{c||}{Recall (Accuracy)} & \multicolumn{2}{c||}{Precision} & \multicolumn{2}{c|}{$F_1$-score} \\
        \hline
        Training Type & task1 & task 2 & task1 & task 2 & task1 & task 2 \\
        \hline
        \hline
        ATE-SG ($E_0=E_{\rm ts}=1$) & 0.916840 & 0.961592 & 0.921696 & 0.962326 & 0.917194 & 0.960846\\
        \hline
        ATE-SG ($E_0=E_{\rm ts}=10$) & 0.922131 & 0.964031 & 0.923422 & 0.964155 & 0.922331 & 0.963479 \\
        \hline
        ATE-SG ($E_0=E_{\rm ts}=100$) & 0.923906 & 0.965613 & 0.925943 & 0.965629 & 0.924248 & 0.965063\\
        \hline
        \hline
        Classic SG & 0.926763 & 0.967181 & 0.927541 & 0.967079 & 0.926777 & 0.966837 \\
        \hline
    \end{tabular}
    }
    \caption{\blue{Real-world dataset. Starting learning rate $10^{-2}$. Performance measures of the MTNNs on the test set $\mathcal{P}$. 
    }}
    \label{tab:rco_performance_1e-2}
\end{table}

\begin{table}[htb!]
    \centering
    \resizebox{1.\textwidth}{!}{
    \begin{tabular}{|c||c|c||c|c||c|c|}
        \hline
        \quad & \multicolumn{2}{c||}{Recall (Accuracy)} & \multicolumn{2}{c||}{Precision} & \multicolumn{2}{c|}{$F_1$-score} \\
        \hline
        Training Type & task1 & task 2 & task1 & task 2 & task1 & task 2 \\
        \hline
        \hline
        ATE-SG ($E_0=E_{\rm ts}=1$) & 0.921241 & 0.964786 & 0.921954 & 0.964736 & 0.921131 & 0.964292\\
        \hline
        ATE-SG ($E_0=E_{\rm ts}=10$) & 0.924411 & 0.965830 & 0.924818 & 0.965756 & 0.924168 & 0.965472 \\
        \hline
        ATE-SG ($E_0=E_{\rm ts}=100$) & 0.923430 & 0.965166 & 0.924457 & 0.965090 & 0.923388 & 0.964671\\
        \hline
        \hline
        Classic SG & 0.923516 & 0.965527 & 0.923931 & 0.965450 & 0.923406 & 0.965076 \\
        \hline
    \end{tabular}
    }
    \caption{\blue{Real-world dataset. Starting learning rate $10^{-3}$. Performance measures of the MTNNs on the test set $\mathcal{P}$. 
    }}
    \label{tab:rco_performance_1e-3}
\end{table}

\quad



\begin{figure}[htb!]
    \centering
    \subcaptionbox{Training loss, starting learning rate $10^{-2}$.}{
    \includegraphics[trim={1.cm 0.25cm 2.45cm 1.78cm},clip,width=.65\textwidth]{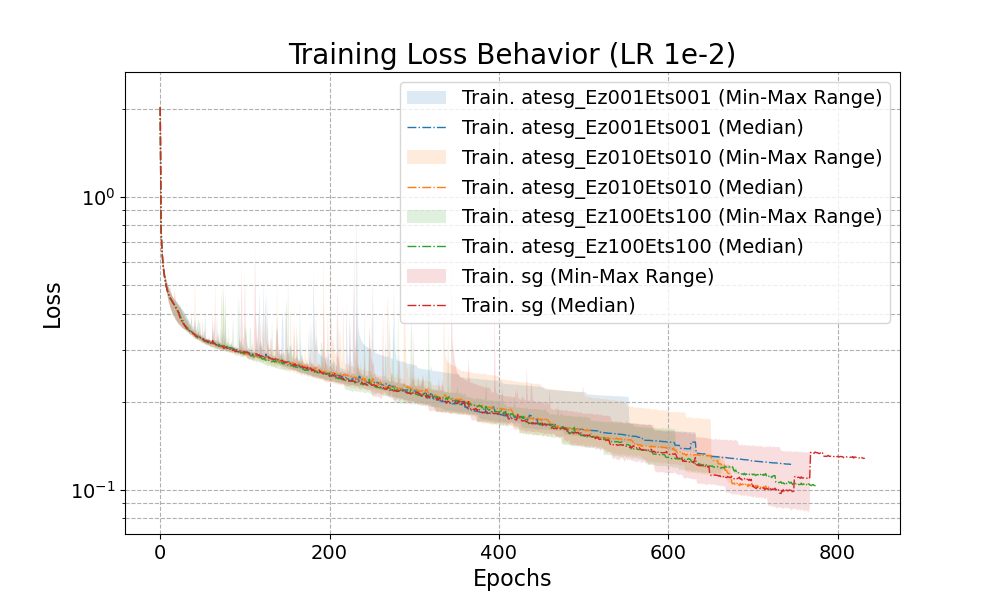}
    }
    \subcaptionbox{Validation loss, starting learning rate $10^{-2}$.}{
    \includegraphics[trim={1.cm 0.25cm 2.45cm 1.75cm},clip,width=.65\textwidth]{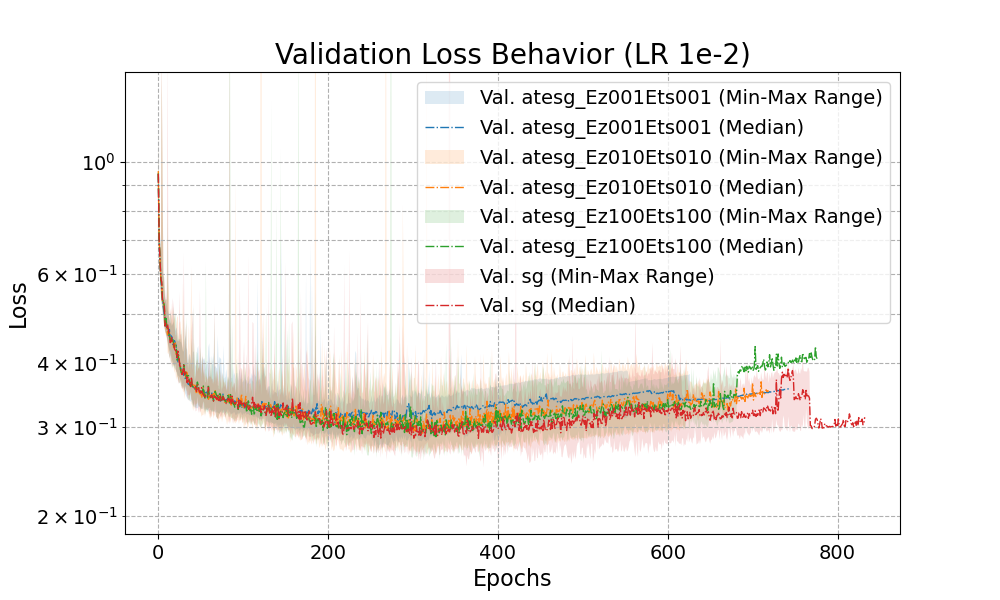}
    }
    \caption{\blue{Real-world dataset case.} 
    }
    \label{fig:rco_losses_1e-2}
\end{figure}

\begin{figure}[htb!]
    \centering
    \subcaptionbox{Training loss, starting learning rate $10^{-3}$.}{
    \includegraphics[trim={1.85cm 0.25cm 2.45cm 1.775cm},clip,width=.65\textwidth]{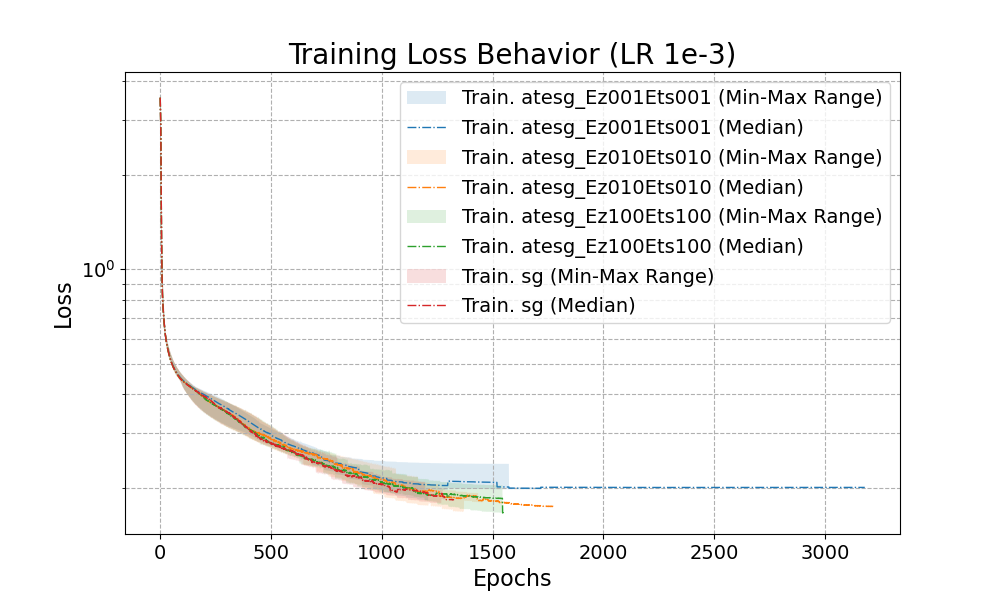}
    }
    \subcaptionbox{Validation loss, starting learning rate $10^{-3}$.}{
    \includegraphics[trim={1.85cm 0.25cm 2.45cm 1.75cm},clip,width=.65\textwidth]{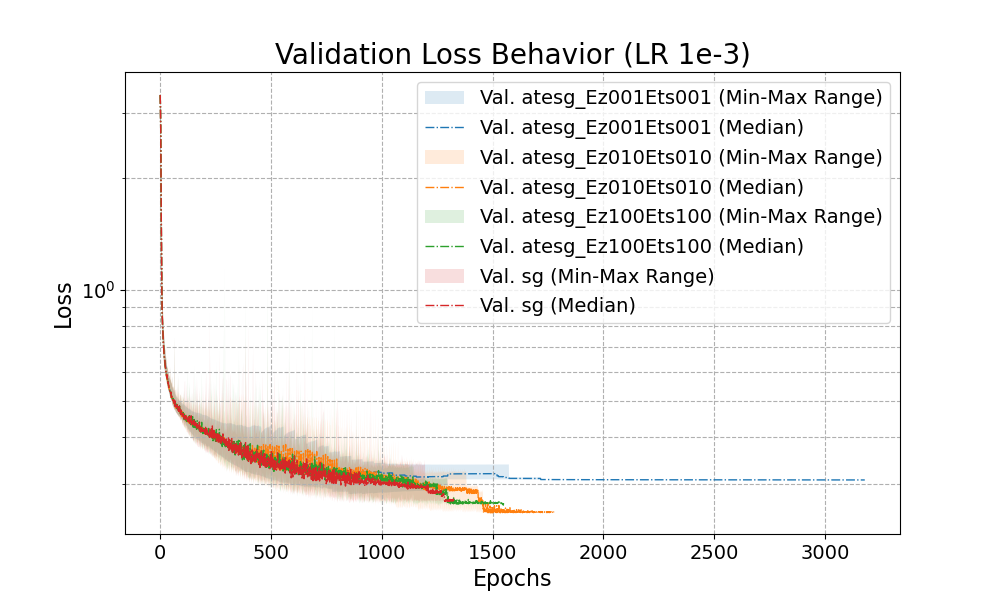}
    }
    \caption{\blue{Real-world dataset case.} 
    }
    \label{fig:rco_losses_1e-3}
\end{figure}

\section{Conclusion}\label{sec:conclusion}

In this work we presented new alternate training procedures for hard-parameter sharing MTNNs. We started illustrating the properties of the task-specific gradients of the loss function of an MTNN, and explaining the motivations behind the proposed alternate method. Then, in \Cref{sec:alterante_training}, we introduced a first formulation of alternate training called Simple Alternate Training (SAT) and a second one called Alternate Training through the Epochs (ATE); both formulations are based on the Stochastic Gradient (SG). For these methods we provided a stochastic convergence analysis.
We concluded the work illustrating the results of two numerical experiments, where the prediction abilities of an MTNN trained using the implemented ATE-SG algorithm are compared with the prediction abilities of the same MTNN trained using the SG. The experiments show very interesting properties of training the network using the implemented ATE-SG; \blue{in particular, in the training phase, we observe a reduction of the computational costs and 
regularization effects on the training (i.e., marked reduction in loss value oscillations), when high-frequency alternation between the shared parameters training phase and the task-specific parameters is adopted.}

In conclusion, the alternate training procedures presented in this work proved to be a novel and useful approach for training MTNNs. In future work, we will analyze how the procedures may change by replacing the stochastic gradient with other optimization methods.
\blue{
Additionally, we will focus on studying the behavior of ATE-SG varying adaptively  the number of epochs in the shared and task-specific  phases along the iterations, taking also into account the learning rate schedule.}




\section*{Acknowledgement(s)}

S.B. acknowledges the support from the Italian PRIN project ``\emph{Numerical Optimization with Adaptive accuracy and applications to machine learning}'', CUP E\-53\-D\-2300\-7690\-006, \emph{Progetti di Ricerca di Interesse nazionale 2022}.
F.D. acknowledges that this study was carried out within the FAIR-Future Artificial Intelligence Research and received funding from the European Union Next-GenerationEU (PIANO NAZIONALE DI RIPRESA E RESILIENZA (PNRR) – MISSIONE 4 COMPONENTE 2, INVESTIMENTO 1.3---D.D. 1555 11/10/2022, PE00000013). This manuscript reflects only the authors’ views and opinions; neither the European Union nor the European Commission can be considered responsible for them. F.D. acknowledges support from Italian MUR PRIN project 20227K44ME, Full and Reduced order modeling of coupled systems: focus on non-matching methods and automatic learning (FaReX).
The authors acknowledge support from INdAM-GNCS Group.

\section*{Disclosure statement}

Authors declare no conflicts of interest







\section*{Notes}

 For a ready-to-use version of the implemented ATE-SG algorithm (see \Cref{alg:practical_ATE}, \Cref{app:implATESG}), visit \url{https://github.com/Fra0013To/ATEforMTNN}.


\bibliographystyle{tfs}
\bibliography{references, aiPapers, DellaPapers, optimizationPapers, otherPapers}


\appendix

\section{Implemented ATE-SG}\label{app:implATESG}
 The pseudo-code of the implemented ATE-SG method is given in \Cref{alg:practical_ATE}. For simplicity, we consider the total number of epochs as the only stopping criterion. 
 For a ready-to-use version of this algorithm, see \url{https://github.com/Fra0013To/ATEforMTNN}.
 
\begin{algorithm}\caption{Implemented ATE-SG}\label{alg:practical_ATE} 
{\small
\begin{description}

\item[Data:] $(\v{w}_0, \v{w}_{\rm ts})=\v{w}$ (initial guesses for the trainable parameters), $\mathcal{T}$ (training set), $B$ (mini-batch size), $\{\eta_i, ~i\ge 0\}$ (learning rates), $\ell$ (loss function),
$E_0,E_{\rm ts}$ (number of epochs for alternate training), $E$ (total number of epochs).

\item[Procedure:] \quad
    \begin{algorithmic}[1]
    \STATE $\v{w}^{(0)} \gets (\v{w}_0, \v{w}_{\rm ts})$ 
    \STATE $e \gets 0$ \quad (epochs counter, total)
    \STATE $i \gets 0$ \quad (iteration counter)
    \WHILE{$e\leq E$}
        
        \STATE $e_0\gets 0$ \quad (epochs counter, shared phase)
        \WHILE{$e_0\leq E_0$ and $e\leq E$ (alternate training, shared phase)}
            \STATE $\{\mathcal{B}_1,\ldots ,\mathcal{B}_t\}\gets$ random split of $\mathcal{T}$ into mini-batches w.r.t. $B$
            \FOR{\blue{$\tau=0,1,\ldots ,t-1$}}
                \STATE $\v{w}_0 \gets \v{w}_0 - \eta_i \nabla_{\v{w}_0}\ell(\mathcal{B}_\tau;\v{w}^{(i)})$
                 \STATE $i \gets i+1$
                \STATE $\v{w}^{(i)} \gets  (\v{w}_0, \v{w}_{\rm ts})$
            \ENDFOR
            \STATE $e_0 \gets e_0 + 1$ and
            $e \gets e + 1$
        \ENDWHILE
        \STATE $e_{\rm ts}\gets 0$ \quad (epochs counter, task-specific phase)
        \WHILE{$e_{\rm ts}\leq E_{\rm ts}$ and $e\leq E$ (alternate training, task-specific phase)}
            \STATE $\{\mathcal{B}_1,\ldots ,\mathcal{B}_t\}\gets$ random split of $\mathcal{T}$ into mini-batches w.r.t. $B$
            \FOR{\blue{$\tau=0,1,\ldots ,t-1$}}
                \STATE $\v{w}_{\rm ts} \gets \v{w}_{\rm ts} - \eta_i \nabla_{\v{w}_{\rm ts}}\ell(\mathcal{B}_\tau; \v{w}^{(i)})$
                \STATE $i \gets i+1$
                \STATE $\v{w}^{(i)} \gets  (\v{w}_0, \v{w}_{\rm ts})$
            \ENDFOR
            \STATE $e_{\rm ts} \gets e_{\rm ts} + 1$ and 
             $e \gets e + 1$
            
        \ENDWHILE
    \ENDWHILE
    \RETURN $\v{w}^{(i)}$ \quad (final MTNN's weights)
    \end{algorithmic}
\end{description}
}
\end{algorithm}

\end{document}